\documentclass[11pt,letterpaper]{article}
\usepackage{url,color,paralist,xspace}            
\usepackage{authblk}
\usepackage{fullpage,times}
\usepackage{booktabs}       
\usepackage{amsfonts}       
\usepackage{nicefrac}       
\usepackage{microtype}      
\usepackage{amsmath, amssymb, amsthm}
\usepackage{graphicx, subfigure}
\usepackage{algorithm}
\usepackage[noend]{algorithmic}
\usepackage{tikz}
\usepackage{multirow}
\usepackage{adjustbox}
\usepackage{wrapfig}
\usetikzlibrary{calc, decorations.pathreplacing}
\def\layersep{1cm}
\tikzset{neuron/.style={circle,thick,fill=black!25,minimum size=12pt,inner sep=0pt},
    input neuron/.style={neuron, draw,thick, fill=gray!50},
    hidden block/.style={circle,thick,fill=black!25,minimum size=25pt,inner sep=0pt,draw,thick, fill=green!50},
    RB block/.style={circle,thick,fill=black!25,minimum size=25pt,inner sep=0pt,draw, thick, fill=blue!50},
    hidden neuron/.style={circle,thick,fill=black!25,minimum size=10pt,inner sep=0pt,draw,thick, fill=green!50},
    RB neuron/.style={circle,thick,fill=black!25,minimum size=10pt,inner sep=0pt,draw, thick, fill=blue!50},
    output neuron/.style={neuron,draw,thick, fill=red!50},
    hoz/.style={rotate=-90}}
\def \tw{\tilde{w}}
\def \tr{\tilde{r}}
\def \tphi{\tilde{\phi}}
\def \mP{\mathbb P}
\def \mR{\mathbb R}
\def \mE{\mathbb E}
\def \mK{\mathcal K}
\def \mM{\mathcal M}
\def \mN{\mathcal N}

\def \mS{\mathcal S}
\def \mF{\mathcal F}

\def \mQ{\mathcal Q}
\def \balpha{{\bf \alpha}}

\def \bphi{\boldsymbol{\phi}}
\def \bPhi{\boldsymbol{\Phi}}
\def \bmu{\boldsymbol{\mu}}
\def \bSigma{\boldsymbol{\Sigma}}
\def \KL{{\bf \rm KL}}
\def \ELBO{{\bf \rm ELBO}}
\def \IP{{\bf \rm IP}}
\def \I{{\bf \rm I}}
\def \b0{{\bf 0}}
\def \bI{{\bf I}}
\def \be{{\bf e}}
\def \bu{{\bf u}}
\def \bU{{\bf U}}
\def \bZ{{\bf Z}}
\def \bz{{\bf z}}

\def \by{{\bf y}}
\def \bX{{\bf X}}
\def \bx{{\bf x}}
\def \bdf{{\bf f}}
\def \bF{{\bf F}}
\def \bw{{\bf w}}
\def \bW{{\bf W}}
\def \bv{{\bf v}}
\def \bV{{\bf V}}
\def \bm{{\bf m}}
\def \bS{{\bf S}}

\def \bz{{\bf z}}

\let\OldS\S
\renewcommand{\S}{\OldS\xspace}
\theoremstyle{definition}
\newtheorem{definition}{Definition}
\newtheorem{theorem}{Theorem}

\newtheorem{proposition}{Proposition}

\title{Building Bayesian Neural Networks with Blocks:\\
On Structure, Interpretability and Uncertainty}
\author[1]{Hao Henry Zhou}
\author[2]{Yunyang Xiong}
\author[2,3]{Vikas Singh}
\affil[1]{Department of Statistics, University of Wisconsin-Madison}
\affil[2]{Department of Biostatistics $\&$ Medical Informatics, University of Wisconsin-Madison}
\affil[3]{Department of Computer Sciences, University of Wisconsin-Madison}
\date{}

\begin{document}
\maketitle
\graphicspath{{figs/}}

\begin{abstract}
  We provide simple schemes to build Bayesian Neural Networks (BNNs), block by block, inspired by a recent
  idea of computation skeletons. We show how by adjusting the types of blocks that are used within the computation skeleton,
  we can identify interesting relationships with Deep Gaussian Processes (DGPs), deep kernel learning (DKL),
  random features type approximation and other topics. We give strategies to
  approximate the posterior via doubly stochastic variational inference for such models which yield uncertainty estimates.
  We give a detailed theoretical analysis and point out extensions that may be
  of independent interest. 
  As a special case, we instantiate our procedure to define a Bayesian {\em additive} Neural network -- 
  a promising strategy to identify statistical interactions and has direct benefits for obtaining 
  interpretable models. 
\end{abstract}

\section{Introduction}\label{sec:intro}

Bayesian Neural Networks (BNNs) generally refer to a class of algorithms that treat neural network models in a Bayesian manner \cite{Her15, Gra11, Nea12}.
Consider the loss function of a well-defined neural network model. 
In optimizing this loss, one often seeks to find a parameter optimum, say $\theta^\ast$ -- 
a point estimate of the weights. The Bayesian
perspective, instead, takes into account the inherent uncertainty in the estimates.
To do so, BNNs introduce priors on the network weights: learning then corresponds to
approximating the posterior,
 i.e., $p(\theta|{\rm data})$ via probabilistic backpropagation \cite{Her15}, 
variational inference \cite{Gra11, Gal16}, expectation propagation \cite{Jyl14} and so on. When uncertainty estimates are important,
as is the case in many applications, BNNs are 
well suited. On the other hand, approximating the posterior is challenging
and further, the choice or design of the prior
may not be straightforward. \\

BNNs 
are motivated by a probabilistic interpretation of
deep neural network learning,
which also underlies a related yet distinct body of
work known as Deep Gaussian Processes (DGPs).
DGPs implement a deep probabilistic {\em non-parametric model} for compositions of functions:
an extension of Gaussian processes,
but with a multi-layer structure \cite{Dam13}.
These models inherit the expressive power of GPs and provide uncertainty estimates through the posterior distribution, similar to BNNs. 
But similar to BNNs, calculating the posterior of DGP is difficult.
Also, DGPs are based on GPs, so 
we must solve for the inverse of the kernel matrix: a problem
for large datasets.
There are recent papers 
devoted to overcoming the difficulties in training DGPs, 
which roughly fall in three classes. {\bf First}, one may incorporate the 
deep neural network structure for designing a ``deep'' kernel for GP \cite{Dan16, Wil16}. 
This offers the benefits of a deep structure but avoids 
the complexity of a deep GP. 
The {\bf second} set of methods come up with a specialized design of BNNs \cite{Gal16, Cut17}: 
using doubly stochastic variational inference, these results 
inherit properties of modern innovations in deep learning
in that they can
use mini-batch training, backpropagation, dropout, and 
automatic differentiation. 
The authors in \cite{Gal16, Cut17} show that such specialized designs of BNNs 
can actually serve as approximations of DGPs--using specific kernels based on the activation functions of interest. But these ideas, from the DGP point of view, only apply to 
a limited class of kernels. A {\bf third} line of attack advocates an approximation for the DGP posterior using ``inducing points'' \cite{Bui16, Dai15, Hen14}. In general, this scheme is suitable 
for {\em any} kernel, but poses challenges because solving for the inverse of kernel matrix in a deep structure is difficult--therefore, 
those methods often make strong independence and Gaussianity assumptions. 
But a recent result in \cite{Sal17} showed -- surprisingly --
that we do {\em not} need to force independence or Gaussianity between 
the layers and the algorithm can, in fact, 
be trained using mini-batch training, doubly stochastic variational inference and backpropagation for large datasets.
Interestingly, we find that 
the reason mini-batch training, backpropagation, and taking the correlations between layers into account, 
also common in the second line of work above, can be shown to work in \cite{Sal17} 
is because the approach actually produces a posterior approximation that belongs to a broad class of BNNs which have a kernel-type structure.
This broad class of BNNs have a nice relationship with DGP, 
maintains correlations between layers and can be trained 
based on doubly stochastic variational inference. 
This class which we identify also nicely ties to recent work on deep kernel learning \cite{Wil16}, deep random features \cite{Dan16} in the first line of work as well as MC dropout \cite{Gal16} and random feature expansion \cite{Cut17}.\\

{\bf Beyond uncertainty estimates: assuming structure on the function class.}
The discussion above focuses only on uncertainty estimates for the parameters of the model.
But in many applications, the interpretability of the model is also 
important. In statistics, to get an interpretable model, we
often impose assumptions on the structure of the function class. For example,
an additive structure may pertain to statistical interactions whereas
a hierarchical structure could help identify the influence of individual-level
effects in a mixed-effects model.
In any case, if the model is simple, then making an assumption of structure (e.g., hierarchical, additive),
yields two benefits: interpretability {\em as well as} uncertainty estimates (based on
certain distributional assumptions).
This suggests that to investigate whether 
the BNN and DGP results can be made more
interpretable, an assumption of structure on the function class may
be a good starting point. 
In fact, within the (non-Bayesian) line of work on DNNs, 
so-called capsule structures \cite{sabour2017dynamic} derive representations
which respect spatial hierarchies between objects. New results \cite{bruna2013invariant} relating 
multi-resolution analysis (MRA) to CNNs have also appeared \cite{mallat2016understanding, angles2018generative}. 
So, can we have
strategies for designing deep neural networks which are fundamentally built with
{\em structure} and {\em interpretability} in mind, but are also easily amenable to {\em uncertainty
estimation}?\\

{\bf Building Bayesian neural networks by blocks.}
Our goal is to incorporate structural assumptions on BNNs.
We describe
a procedure based on the so-called ``computation skeleton'' idea in \cite{Dan16}, used
to study the relationship between 
neural networks and kernels.
Here, our rationale for using computation skeletons is to design a BNN, block by block, that can efficiently approximate the
posterior of DGPs. This scheme offers the flexibility to choose between different structures and/or
uncertainty estimation schemes.
It retains all useful empirical properties
such as mini-batch training, works on large-scale datasets and yields the 
expressive power of DGPs with kernels.
Our results also provide a deeper understanding of the relation between BNN and DGP.\\

{\bf Structural assumption and an example case for interpretability.} 
With the computation skeleton framework for BNN in hand, it will be 
easy to design approximation schemes for the DGP posterior.
While the {\em uncertainty of prediction accuracy}, has been studied 
for DNNs \cite{Cut17, Sal17, Gal16}, 
another quantity, {\em statistical interactions}, 
important for interpretability, 
has received little attention \cite{Tsa17}. 
If an output depends on several features,
one is often interested in changing some features to evaluate how it affects the response. 
In doing so, we must guarantee that other ``uncontrolled'' features do {\em not} influence the response. 
This confound is called {\em interaction}: the simultaneous influence of several features on the outcome is {\em not additive} and
the features may jointly affect the outcome.
Interpretability means understanding how predictors influence the outcome. But failing to detect statistical interactions causes
problems in
inferring the features' influence (e.g., the Simpson's paradox). 
A general network architecture permits
{\em all} features to interact, without the ability to control for the nuisance terms.
In statistics, we may use a fully additive statistical model with ANOVA decomposition.
Similarly, we propose an additive structure on the network and apply post-training ANOVA decomposition to detect statistical interactions.
Describing how the neural network architecture is built, with blocks, additively, is where the computation graph idea is essential --
it yields a Bayesian 
{\em additive} neural network (BANN).
Note that statistical interactions 
are different from other interpretability focused results in
computer vision, specifically on {\em relevance} or {\em attribution}\cite{Zei14, Shi17learning, Anc18towards, Sun17axiomatic}.
For example, gradient-based methods provide pixel-importance salience maps.
While {\em local} attribution, which most works focus on, describes how the network response
changes when we infinitesimally perturb the input sample, {\em global} attribution captures the marginal effect of a feature on the
network output with respect to a baseline. We will show that global attribution can be obtained from our scheme as well.

\subsection{Preliminaries}\label{sec:pre}
In this section, we briefly review Deep Gaussian process and variational inference schemes that are often used to approximate the posterior distribution
to setup the rest of our presentation.\\

{\bf Gaussian processes (GP) and deep Gaussian processes (DGPs).}
Consider the inference task for a stochastic function $f:\mR^p\rightarrow \mR$, given a likelihood $p(y|f)$ and a set of $n$ observations
$\by = (y_1,...,y_n)^T\in \mR^{n\time 1}$ at locations $\bX = (\bx_1,...,\bx_n)^T\in \mR^{n\times p}$.
We place a GP prior on the function $f$ that models all function values as jointly Gaussian, with a 
covariance $\mK:\mR^p\times\mR^p\rightarrow \mR$. We use the notation $\bdf = f(\bX)$ and $\mK(\bX, \bX)_{ij} = \mK(\bx_i, \bx_j)$. Then, the joint density for $\by$ and $\bdf$
for a single-layer Gaussian process (GP) is 

$$p(\by,\bdf)=p(\bdf;\bX)\prod_{i=1}^n p(y_i|f_i),$$

where $\bdf|\bX\sim N(0, \mK(\bX,\bX))$ and $y_i|f_i\sim N(f_i,\delta^2)$.\\
%

Then for $L$ vector-valued stochastic functions denoted as $\mF^\ell$, a Deep Gaussian Process (DGP) \cite{Dam13} defines a prior recursively on $\mF^1,...,\mF^L$. 
The prior on each function $\mF^\ell$ is an independent GP in each dimension, with input locations given by the function values at the previous layer: the outputs of GPs at layer $\ell$ are $\{\bF^\ell_{.j}\}_{j=1}^d$ and the corresponding inputs are $\bF^{\ell-1}$.
The joint density of the process is 

$$p(\by,\{\bF^\ell\}_{\ell=1}^L)=\prod_{i=1}^n p(y_i|f^L_i)\prod_{\ell=1}^L p(\bF^\ell|\bF^{\ell-1}),$$

where $\bF^0=\bX$, $\bF^\ell\in \mR^{n\times d_\ell}$ for $0<\ell\leq L$. Here, $\bF^\ell_{.j}|\bF^{\ell-1}\sim N(0,\mK^{\ell}_j(\bF^{\ell-1},\bF^{\ell-1}))$ for $1\leq j \leq d_\ell$, $0<\ell\leq L$. \\

{\bf Variational inference (VI) for Bayesian models.}
Consider the joint density of the latent variables $\bdf = \{f_i\}_{i=1}^m$ and the observations $\by = \{y_i\}_{i=1}^n$, 

$$p(\bdf,\by) = p(\bdf)p(\by|\bdf).$$ 

We know that inference in any Bayesian model amounts to conditioning on the data and computing the posterior $p(\bdf|\by)$. In models like DGP, this calculation is difficult and so, we use approximate inference.
A popular strategy is variational inference (VI) \cite{Ble17} which requires 
specifying a family of {\em approximate} densities $\mQ$.
Our goal is to find the member $q^\ast$ of that family which minimizes the Kullback-Leibler (KL) divergence to the exact posterior, 

$$q^\ast(\bdf) = \arg\min_{q(\bdf)\in\mQ} \KL(q(\bdf)||p(\bdf|\by)).$$ 

Instead of minimizing the KL divergence,
one maximizes the evidence lower bound (ELBO), 

$$\ELBO(q) = \mE_{q(\bdf)}[\log p(\by|\bdf)] - \KL(q(\bdf||p(\bdf)).$$
%

The first term is an expected likelihood, which encourages the densities to place their mass on configurations of the latent variables which explain the observed data. The second term is the negative KL divergence between the variational density and the prior, which encourages densities to lie close to the prior.
For DGP type models, VI is a preferred strategy to approximate the posterior. 

\section{Building Bayesian neural networks: computation skeleton and blocks}\label{sec:skeleton}
We first define the computation skeleton and show how it can lead to a BNN. Then, we show that the constructed BNNs can be seen
as a VI approximation for the DGP posterior. 
Finally, we discuss how to reconstruct other DGP approximations or BNNs \cite{Gal16, Wil16, Dan16} in our framework. \\


{\bf What is a computation skeleton?}
Computation skeleton \cite{Dan16} is a structure to compactly describe a feed-forward computation 
structure from the inputs to the outputs. 
Formally, a computation skeleton $\mS$ is a multi-layer graph with the bottom nodes representing inputs, the 
top nodes representing outputs and non-input nodes are labeled by activations $\sigma$. 
In \cite{Dan16}, this idea was used to study a family 
of DNNs and their properties: it was shown that DNNs can be seen 
as the realization of certain types of structures and their dual kernels.
In fact, every $\mS$ defines a specific NN structure: 
Fig. \ref{fig:blocks}(a) shows a two layer fully connected NN, see \cite{Dan16} for more examples.
Here, we reuse the name but define a slightly different $\mS$ to design BNNs. For notational simplicity, we consider $\mS$'s with a single output.\\

{\bf What are blocks?}
To construct BNNs from a computation skeleton $\mS$, we also need two additional components, which we call ``blocks''. 
Our first type of block is a {\bf function block} denoted as $FB(\mP_\bv, r, d)$, which allows every ``node'' in 
$\mS$ to replicate $d$ times. This will help us in defining Bayesian priors and posteriors. 
We setup FB as a one layer NN where the inputs nodes and output nodes are fully connected.
All incoming edges to the output node $f_j$ as in Fig. \ref{fig:blocks}(b) form a vector $\bv_j$.
The set of $\bv_j$'s for $1\leq j\leq d$ i.i.d. follow the 
distribution $\mP_\bv$ on $\mR^r$. 
$FB(\mP_\bv, r, d)$ simply takes the inputs $\bphi = (\phi_1,...,\phi_r)^T$ and outputs a $d$-dimension vector $\bdf$ with $f_j = \bphi^T \bv_j$ for $1\leq j\leq d$. 
Our second type of block is a {\bf random feature block} denoted as 
$RB(\mP_\bw, d, r, \sigma_{\mK})$, which we use to construct random feature approximations 
for kernels to leverage the expressive power of DGP. 
We setup RB as a one layer NN with random weights where the inputs nodes and outputs are fully connected. 
All incoming edges to the output node $\phi_j$ as in Fig. \ref{fig:blocks}(c) 
form a vector $\bw_j$. The set of $\bw_j$'s for $1\leq j\leq r$ 
follow the distribution $\mP_\bw$ on $\mR^d$ for $1\leq j\leq d$. 
$RB(\mP_\bw, d, r, \sigma_{\mK})$ takes the inputs $\bx = (x_1,...,x_d)^T$ and outputs a 
$r$-dimension vector $\bphi$ with $\phi_j = \frac{1}{\sqrt{r}}\sigma_{\mK}(\bx^T\bw_j)$ for $1\leq j\leq r$.\\

\begin{figure}
	\centering
	\subfigure[][$\mS$]{
		\resizebox{.22\linewidth}{2cm}{
		\begin{tikzpicture}[-,draw=black, very thick, node distance=\layersep,transform shape,rotate=90]  
\tikzstyle{annot} = [text width=4em, text centered]

\foreach \name / \y in {1/1,2/2,3/3,4/4}
\node[input neuron, hoz] (I-\name) at (0,-\name) {};


\foreach \name / \y in {1/1}
\node [hidden neuron, hoz] (H-\name) at (\layersep,-\name-1.5) {}; 
\node[hoz] (Hsigma) at (\layersep,-1.5) {\Large{$\sigma$}};
\foreach \name / \y in {1/1}
\node [output neuron, hoz] (J-\name) at (2*\layersep,-\name-1.5) {};
\node[hoz](Jsigma) at (2*\layersep,-1.5) {\Large{$\sigma$}};
\foreach \source in {1,2,3,4}
	\foreach \dest in {1}
		\path (I-\source.north) edge (H-\dest.south);
		
\foreach \source in {1}
	\foreach \dest in {1}
		\path (H-\source.north) edge (J-\dest.south);
		
\end{tikzpicture}}}
	\subfigure[][$FB$]{
		\resizebox{.22\linewidth}{2cm}{
		\begin{tikzpicture}[-,draw=black, very thick, node distance=\layersep,transform shape,rotate=90]  
\tikzstyle{annot} = [text width=4em, text centered]

\foreach \name / \y in {1/1,2/2,3/3}
\node[RB block, hoz] (I-\name) at (0,-\name*1.5) {\Large{$\phi_\y$}};


\foreach \name / \y in {1/1, 2/2}
\node [hidden block, hoz] (H-\name) at (2*\layersep,-\name*1.5-0.5*1.5) {\Large{$f_\y$}}; 



\foreach \source in {1,2,3}
	\foreach \dest in {1,2}
		\path (I-\source.north) edge (H-\dest.south);


\end{tikzpicture}}}
	\subfigure[][$RB$]{
		\resizebox{.22\linewidth}{2cm}{
		\begin{tikzpicture}[-,draw=black, very thick, node distance=\layersep,transform shape,rotate=90]  
\tikzstyle{annot} = [text width=4em, text centered]

\foreach \name / \y in {1/1,2/2}
\node[hidden block, hoz] (I-\name) at (0,-\name*1.5-0.5*1.5) {\Large{$x_\y$}};


\foreach \name / \y in {1/1, 2/2, 3/3}
\node [RB block, hoz] (H-\name) at (2*\layersep,-\name*1.5) {\Large{$\phi_\y$}}; 



\foreach \source in {1,2}
	\foreach \dest in {1,2,3}
		\path (I-\source.north) edge (H-\dest.south);


\end{tikzpicture}}}
	\subfigure[][BNN $\mN(\mS)$]{
		\resizebox{.22\linewidth}{2cm}{
		\begin{tikzpicture}[-,draw=black, very thick, node distance=\layersep,transform shape,rotate=90]  
\tikzstyle{annot} = [text width=4em, text centered]

\def\layersep{0.75cm}
\def\nodesep{1cm}
\def\leftsidelabels{0.9}

\foreach \name / \y in {1/1,2/2,3/3,4/4}
\node[input neuron, hoz] (I-\name) at (0,-\name*\nodesep) {};



\foreach \name / \y in {1/1, 2/2, 3/3}
\node [RB neuron, hoz] (H-\name) at (0.75*\layersep,-\name*\nodesep-0.5*\nodesep) {}; 
\node[hoz] (Hsigma) at (0.75*\layersep,-\leftsidelabels) {\Large{$\sigma_{\mK}$}};

\foreach \name / \y in {1/1, 2/2}
\node [hidden neuron, hoz] (J-\name) at (1.5*\layersep,-\name*\nodesep-1*\nodesep) {};
\node[hoz] (Jsigma) at (1.5*\layersep,-\leftsidelabels) {\Large{$\sigma$}};

\foreach \name / \y in {1/1, 2/2, 3/3}
\node [RB neuron, hoz] (K-\name) at (2.25*\layersep,-\name*\nodesep-0.5*\nodesep) {}; 
\node[hoz] (Ksigma) at (2.25*\layersep,-\leftsidelabels) {\Large{$\sigma_{\mK}$}};

\foreach \name / \y in {1/1}
\node [output neuron, hoz] (L-\name) at (3*\layersep,-\name*\nodesep-1.5*\nodesep) {};
\node[hoz] (Lsigma) at (3*\layersep,-\leftsidelabels) {\Large{$\sigma$}};

\foreach \source in {1,2,3,4}
	\foreach \dest in {1,2,3}
		\path (I-\source.north) edge (H-\dest.south);
		
\foreach \source in {1,2,3}
	\foreach \dest in {1,2}
		\path (H-\source.north) edge (J-\dest.south);
		
\foreach \source in {1,2}
	\foreach \dest in {1,2,3}
		\path (J-\source.north) edge (K-\dest.south);
		
\foreach \source in {1,2,3}
	\foreach \dest in {1}
		\path (K-\source.north) edge (L-\dest.south);
		
\node[hoz] at (0.75*\layersep,-4.2){\Large{$\phi^1$}};
\node[hoz] at (1.5*\layersep,-4.2){\Large{$f^1$}};
\node[hoz] at (2.25*\layersep,-4.2){\Large{$\phi^2$}};
\node[hoz] at (3*\layersep,-4.2){\Large{$f^2$}};
\end{tikzpicture}}}
	\subfigure[][MC dropout]{
		\resizebox{.22\linewidth}{2cm}{
		\begin{tikzpicture}[-,draw=black, very thick, node distance=\layersep,transform shape,rotate=90]  
\tikzstyle{annot} = [text width=4em, text centered]

\foreach \name / \y in {1/1,2/2,3/3,4/4}
\node[input neuron, hoz] (I-\name) at (0,-\name) {};


\foreach \name / \y in {1/1,2/2}
\node [hidden neuron, hoz] (H-\name) at (\layersep,-\name-1) {}; 
\node[hoz] (Hsigma) at (\layersep,-1.5) {\Large{$\sigma$}};
\foreach \name / \y in {1/1,2/2}
\node [output neuron, hoz] (J-\name) at (2*\layersep,-\name-1) {};
\node[hoz](Jsigma) at (2*\layersep,-1.5) {\Large{$\sigma$}};
\foreach \source in {1,2,3,4}
	\foreach \dest in {1,2}
		\path (I-\source.north) edge (H-\dest.south);
		
\foreach \source in {1,2}
	\foreach \dest in {1,2}
		\path (H-\source.north) edge (J-\dest.south);

\draw[fill=white, fill opacity=0.8] (0.25,-0.5) rectangle (0.75,-4.5) node[fill opacity =1,hoz,pos=.5] {\bf FB};
\draw[fill=white, fill opacity=0.8] (1.25,-0.5) rectangle (1.75,-4.5) node[fill opacity =1,hoz,pos=.5] {\bf FB};
		
\end{tikzpicture}}}
	\subfigure[][DRF]{
		\resizebox{.22\linewidth}{2cm}{
		\begin{tikzpicture}[-,draw=black, very thick, node distance=\layersep,transform shape,rotate=90]  
\tikzstyle{annot} = [text width=4em, text centered]

\foreach \name / \y in {1/1,2/2,3/3,4/4}
\node[input neuron, hoz] (I-\name) at (0,-\name) {};


\foreach \name / \y in {1/1}
\node [output  neuron, hoz] (H-\name) at (2*\layersep,-\name-1.5) {}; 
\node[hoz] (Hsigma) at (2*\layersep,-1.5) {\Large{$\sigma$}};
\foreach \source in {1,2,3,4}
	\foreach \dest in {1}
		\path (I-\source.north) edge (H-\dest.south);
		

\draw[fill=white, fill opacity=0.8] (0.25,-0.5) rectangle (0.75,-4.5) node[fill opacity =1,hoz,pos=.5] {\bf RB};
\draw[fill=white, fill opacity=0.8] (0.75,-0.5) rectangle (1.25,-4.5) node[fill opacity =1,hoz,pos=.5] {\bf RB};
\draw[fill=white, fill opacity=0.8] (1.25,-0.5) rectangle (1.75,-4.5) node[fill opacity =1,hoz,pos=.5] {\bf FB};

\end{tikzpicture}}}
	\subfigure[][DKL]{
		\resizebox{.22\linewidth}{2cm}{
		\begin{tikzpicture}[-,draw=black, very thick, node distance=\layersep,transform shape,rotate=90]  
\tikzstyle{annot} = [text width=4em, text centered]

\foreach \name / \y in {1/1,2/2,3/3,4/4}
\node[input neuron, hoz] (I-\name) at (0,-\name) {};



\foreach \name / \y in {1/1}
\node [hidden neuron, hoz] (H-\name) at (\layersep,-\name-1.5) {}; 
\node[hoz] (Hsigma) at (\layersep,-1) {\Large{$\sigma$}};
\foreach \name / \y in {1/1,2/2,3/3}
\node [output neuron, hoz] (J-\name) at (2*\layersep,-\name-0.5) {};
\node[hoz](Jsigma) at (2*\layersep,-1) {\Large{$\sigma$}};
\foreach \source in {1,2,3,4}
	\foreach \dest in {1}
		\path (I-\source.north) edge (H-\dest.south);
		
\foreach \source in {1}
	\foreach \dest in {1,2,3}
		\path (H-\source.north) edge (J-\dest.south);

\draw[fill=white, fill opacity=0.8] (0.25,-0.5) rectangle (0.75,-4.5) node[fill opacity =1,hoz,pos=.5] {\textbf{FB}};
\draw[fill=white, fill opacity=0.8] (1.25,-0.5) rectangle (1.75,-4.5) node[fill opacity =1,hoz,pos=.5] {{\bf FB + IPB}};

\end{tikzpicture}}	}
	\subfigure[][RF, additive]{
		\resizebox{.22\linewidth}{2cm}{
		  \begin{tikzpicture}[-,draw=black, very thick, node distance=\layersep,transform shape,rotate=90]  
\tikzstyle{annot} = [text width=4em, text centered]

\foreach \name / \y in {1/1,2/2,3/3,4/4}
\node[input neuron, hoz] (I-\name) at (0,-\name*0.5) {};

\foreach \name / \y in {1/1,2/2,3/3,4/4}
\node[input neuron, hoz] (P-\name) at (0,-\name*0.5-4*0.5) {};




\foreach \name / \y in {1/1,2/2}
\node [hidden neuron, hoz] (H-\name) at (0.67*\layersep,-4*\name*0.5+1.5*0.5) {}; 
\node[hoz] (Hsigma) at (0.67*\layersep,-1.5*0.5) {\Large{$\sigma$}};

\foreach \name / \y in {1/1,2/2}
\node [hidden neuron, hoz] (J-\name) at (0.67*2*\layersep,-4*\name*0.5+1.5*0.5) {};
\node[hoz](Jsigma) at (0.67*2*\layersep,-1.5*0.5) {\Large{$\sigma$}};

\foreach \name / \y in {1/1}
\node [output neuron, hoz] (Q-\name) at (0.67*3*\layersep,-4.5*0.5) {}; 

\foreach \source in {1,2,3,4}
	\foreach \dest in {1}
		\path (I-\source.north) edge (H-\dest.south);
		
\foreach \source in {1,2,3,4}
	\foreach \dest in {2}
		\path (P-\source.north) edge (H-\dest.south);
		
\foreach \source in {1}
	\foreach \dest in {1}
		\path (H-\source.north) edge (J-\dest.south);
		
\foreach \source in {2}
	\foreach \dest in {2}
		\path (H-\source.north) edge (J-\dest.south);
		
\foreach \source in {1,2}
	\foreach \dest in {1}
		\path (J-\source.north) edge (Q-\dest.south);
		
\end{tikzpicture}}}
	\vspace{-10pt}
	\caption{ \label{fig:blocks} Using $\mS$ in (a), one can construct a BNN $\mN(\mS)$ in (d) with the function block ($FB$) in (b) and
          the random feature block ($RB$) in (c). For different $\mS$, (e) is MC dropout, (f) is deep random features, (g) is deep kernel learning. (g) also represents a multi-task $\mS$ while (h) gives the additive structure.}
\end{figure}

{\bf Constructing a BNN with $\mS$, FB and RB blocks.}
Let us denote $s^\ell$ to be the number of nodes in layer $\ell$ of the computation skeleton $\mS$. Typically, we may 
choose $\mP_\bv \sim N(\bmu, \bSigma)$ and $\mP_\bw \sim \rho N(\b0, \bI)$ for a constant $\rho$.
Alg. \ref{alg:constructBNN} shows how given a $\mS$, together with $FB$ and $RB$ blocks, we
can construct a BNN $\mN(\mS)$ by sequentially replacing edges in $\mS$ with a combination of 
$FB$ and $RB$ from bottom (input nodes) up to the top (output nodes). 
Shortly, we describe the properties of such a BNN. First, let us see an example. 
For Fig. \ref{fig:blocks}, using $FB$ in (b) with $r=3$ and $d=2$ ($d=1$ for the last layer)
and $RB$ in (c) with $d=2$ and $r=3$ ($d=4$ for the first layer) in Algorithm \ref{alg:constructBNN}, we construct a BNN in (d) from the 
in $\mS$ in (a). Essentially, we substitute in $FB+RB$ to 
replace every edge in Fig. \ref{fig:blocks}(a).


\begin{algorithm}[!b]
\caption{\label{alg:constructBNN} Constructing a Bayesian neural network (BNN) with computation skeleton and blocks}
	\begin{algorithmic}
	\STATE{{\bf Input:} a computation skeleton $\mS$. {\bf Output:} a deep BNN $\mN(\mS)$.}
	\STATE{Construct layer $0$ in $\mN(\mS)$ by copying inputs (layer $0$) from $\mS$.}
	\FOR{$\ell = 1$ to $L$}
	\STATE{$\bdf^{\ell-1} = (\bdf^{\ell-1}_1;...;\bdf^{\ell-1}_{s^{\ell-1}})\in \mR^{d^{\ell-1}}$: output vector on layer $\ell-1$ in $\mN(\mS)$. }
	\STATE{For each $\bdf^{\ell-1}_j$, $1\leq j \leq s^{\ell-1}$, apply the activation $\sigma$ in $\mS$, and output $\{\sigma(\bdf^{\ell-1}_j)\}_{j=1}^{s^{\ell-1}}$.}
	\FOR{$i = 1$ to $s^\ell$}
	\STATE{$In(i)=\{1\leq j\leq s^{\ell-1}|$ if node $j$ in layer $\ell-1$ connects with node $i$ in layer $\ell$ in $\mS\}$}
	\STATE{Build $RB(\mP_{\bw^\ell_i}, d^{\ell-1}, r, \sigma_{\mK})$ on $\{\sigma(\bdf^{\ell-1}_j)\}_{j\in In(i)}$ and output $\bphi^\ell_i\in \mR^r$.}
	\STATE{Build $FB(\mP_{\bv^\ell_i}, r, d^\ell_i)$ on $\bphi^\ell_i$ and output $\bdf^{\ell}_i\in \mR^{d^\ell_i}$ in layer $\ell$ of $\mN(\mS)$}	
	\ENDFOR
	\ENDFOR
	\end{algorithmic}
\end{algorithm}
\subsection{Prior and posterior approximation for $\mN(\mS)$} 
Our remaining task is 
to describe a prior for $\mN(\mS)$ and then derive a posterior approximation scheme 
for the construction in Alg. \ref{alg:constructBNN}. 
To do so, we define some notations. 
We use $\bW$ for all random weights in the $RB$ blocks, 
$\bV$ gives all BNN weights in the $FB$ blocks and $\bv^\ell_k$ denotes the weight vector 
that goes into $k$th dimension of $\bdf^\ell$. The related random features are denoted by $\bphi^\ell_k$. 
For $\bX = (\bx_1,...,\bx_n)^T$, we denote $\bF^\ell$ as a matrix with the $i$th row $\bF^\ell_{i.}=\bdf^\ell(\bx_i)$ 
as the value of $\bdf^\ell$ evaluated on input $\bx_i$. 
We define $\bPhi^\ell_k$ to be the random feature matrix related to $\bdf^\ell_k$ for $1\leq k\leq d^\ell$.

\begin{definition}\label{def:BNN}
For a BNN $\mN(\mS)$ from Algorithm \ref{alg:constructBNN}, 
we treat $\bW$ as fixed, then the parameters are only $\bV$. We choose $\mP_\bv = N(\b0, \bI_r)$ in Algorithm \ref{alg:constructBNN} to define the 
Bayesian prior on $\bv^\ell_k$ as 
$N(\b0, \bI_r)$ for 
  $1\leq k \leq d^\ell$, $1\leq \ell\leq L$. This Bayesian prior leads to the relation 
  $p(\bF^\ell_{.k}|\bF^{\ell-1})=N(\bF^\ell_{.k};\b0, \bPhi^\ell_k{\bPhi^\ell_k}^T)$, therefore has a distribution over 
  $\{\bF^\ell\}_{\ell=1}^L$ which is 
  $p(\{\bF^\ell\}_{\ell=1}^L) = \prod_{\ell=1}^L p(\bF^\ell|\bF^{\ell-1})$.
\end{definition}

When the outputs $\by$ and likelihood $p(\by|\bF^L)$ are available for the design matrix $\bX$, the posterior of BNN $\mN(\mS)$ is intractable.
Therefore, we use variational inference to approximate its posterior. We define the variational inference approximation for the posterior of $\bV$ in $\mN(\mS)$ by defining 
the variational posterior $q$ over $\bV$ with $\bv^\ell_k\sim N(\bmu^\ell_k, \bSigma^\ell_k)$. 
Then, we get the $\ELBO$

\begin{equation}
\ELBO = \sum_{i=1}^n \mE_{q(f^L_i)}\log(p(y_i|f_i)) - \KL(q(\bV)|p(\bV))
\end{equation}

This variational posterior over $\bV$ also leads to a posterior over $\{\bF^\ell\}_{\ell=1}^L$.
We apply a doubly stochastic approximation for the first term in the ELBO, where the sum is estimated using mini-batches and the expectation is approximated with a Monte Carlo sample from the variational posterior $q(f^L_i)$. Both stochastic approximations are unbiased. Further, by reparameterizing $\bv^\ell_k = \bmu^\ell_k + {\bSigma^\ell_k}^{1/2}N(\b0,\bI_r)$, the optimization of ELBO 
can be achieved with mini-batch training and backpropagation\cite{Cut17, Sal17, Gal16}. 

\subsection{Relationship of $\mN(\mS)$ to approximate deep Gaussian processes (DGPs)}
Having constructed a BNN $\mN(\mS)$ from $\mS$, we can 
study the relationship 
between $\mN(\mS)$ and DGP. We will show that $\mN(\mS)$ from Alg. \ref{alg:constructBNN} is a 
VI approximation for a DGP posterior. 
To simplify notation, we assume that all $\{\bPhi^\ell_k\}_{k=1}^{d^\ell}$ are the 
same so we drop the subscript $k$. We also assume that all $\{d^\ell\}_{\ell=1}^L$ are the same.
We define an empirical kernel and its expectation as

\begin{equation}\begin{aligned}\label{eq:DGPkernel}
&\text{[Empirical] }\quad \hat{\mK}^\ell(\bdf^{\ell-1}(\bx),\bdf^{\ell-1}(\bx')) = \frac{1}{r}\sum_{i=1}^r\sigma_{\mK}(\sigma(\bdf^{\ell-1}(\bx))^T\bw_i)\sigma_{\mK}(\sigma(\bdf^{\ell-1}(\bx'))^T\bw_i).\\
&\text{[Expectation] }\quad \mK^\ell(\bdf^{\ell-1}(\bx),\bdf^{\ell-1}(\bx'))=\mE_{\bw}\sigma_{\mK}(\sigma(\bdf^{\ell-1}(x))^T\bw)\sigma_{\mK}(\sigma(\bdf^{\ell-1}(x'))^T\bw.
\end{aligned}\end{equation}

It is easy to check that $\hat{\mK}^\ell(\bdf^{\ell-1}(\bx),\bdf^{\ell-1}(\bx')) = \langle \bphi^\ell(\bx), \bphi^\ell(\bx')\rangle$. We denote $\hat{\mK}^\ell(\bF^{\ell-1},\bF^{\ell-1})$ as the $n\times n$ matrix for $n$ inputs. We point out that 
the prior in Definition \ref{def:BNN} is indeed a DGP prior.
\begin{proposition}\label{prop:DGP}
The BNN prior of $\mN(\mS)$ in Def. \ref{def:BNN} gives a DGP prior for $\{\bF^\ell\}_{\ell=1}^L$. This means that $\bF^\ell_{.j}|\bW,\bF^{\ell-1}\sim N(0,\hat{\mK}^\ell(\bF^{\ell-1},\bF^{\ell-1})),$
for $1\leq j\leq d$ and $1\leq \ell\leq L$.
\end{proposition}
We can also show that the kernels $\{\hat{\mK}^\ell\}_{\ell=1}^L$ for this DGP is close to 
the kernel $\{\mK^\ell\}_{\ell=1}^L$ in \eqref{eq:DGPkernel} if $\sigma_{\mK}$ is ReLU or $C$-bounded, 
i.e., $\sigma_{\mK}$ is continuously differentiable and $||\sigma_{\mK}||_\infty, ||\sigma_{\mK}'||_\infty \leq C$.

\begin{theorem}\label{thm:kernel}
If the activation function $\sigma_{\mK}$ is ReLU, then for every $1\leq \ell \leq L$, on a compact set $\mM \in \mR^d$ with diameter $diam(\mM)$ and $\max_{\Delta\in \mM}||\Delta||_2\leq c_\mM$, with probability at least $1-c_1c_{\mM} diam(\mM)^2\exp\left\{-\frac{r\epsilon^2}{8(1+d)\nu_{\mM}^2}\right\}$,

\begin{equation*}
\sup_{\sigma(\bdf^{\ell-1}(\bx)), \sigma(\bdf^{\ell-1}(\bx'))\in \tilde{\mM}}
|\hat{\mK}^\ell(\bdf^{\ell-1}(\bx),\bdf^{\ell-1}(\bx'))-\mK^\ell(\bdf^{\ell-1}(\bx),\bdf^{\ell-1}(\bx'))| \leq \epsilon,
\end{equation*}

for a constant $c_1> 0$ and a parameter $\nu_\mM$ depending on $\mM$. Here, $\tilde{\mM}$ 
specifies that we require $\sigma(\bdf^{\ell-1}(\bx))$ and $\sigma(\bdf^{\ell-1}(\bx'))$ to be two vectors in $\mM$ 
that are not collinear.
\end{theorem}

{\bf C-boundedness.} The uniform concentration bound for $C$-bounded activation functions is given 
in the supplement. The $C$-bounded condition holds for most of the popular sigmoid-like functions such as $1/(1+e^{-x}), erf(x), x/\sqrt{1+x^2}, \tanh(x)$ and $\tan^{-1}(x)$.\\

{\bf Remark 1.} 
In \cite{Dan16}, the authors show that the type of kernels constructed from \eqref{eq:DGPkernel} includes linear, polynomial, arc-cosine, radial basis kernels and so on. For ReLU, the authors in \cite{Cho09, Cut17, Dan16} point out that $\mK^\ell$ is the arc-cosine kernel and \cite{Dan16} shows that $\hat{\mK}$ is a sub-exponential random variable.\\


Since we show that $\{\bF^\ell\}_{\ell=1}^L$ in $\mN(\mS)$ can be seen as generated from a DGP based on $\hat{\mK}$, we can
use many DGP-based approaches to approximate the posterior. Our construction of $\mN(\mS)$ until now is close to the random feature approximation for DGP \cite{Cut17}
except that we allow $\sigma$ to be activation functions (instead of just the identity). Next, we show that using \cite{Sal17} based on inducing points, one also gets the same variational posterior as ours for BNN $\mN(\mS)$. This result enables us to extend our $\mN(\mS)$ construction to be applicable to {\em any kernel class} and it also implies the underlying connection between random feature \cite{Cut17} and inducing points approximations \cite{Sal17} for DGP. First, we apply the inducing points method in \cite{Sal17} on $\mN(\mS)$ with $\hat{\mK}$ to obtain an approximate posterior.

%
%

\begin{theorem}\label{thm:sDGP}
Using the variational approximation \cite{Sal17} for the posterior of a DGP defined on $\{\hat{\mK}^\ell\}_{\ell=1}^L$ with inducing points, we obtain exactly the same variational posterior $q(\{\bF^\ell\}_{\ell=1}^L)$ and evidence lower bound $\ELBO$ 
  as the variational posterior for $\mN(\mS)$. 
\end{theorem}

The result tells us that the random feature expansion in \cite{Cut17} and inducing points method \cite{Sal17} 
are {\em equivalent} for DGPs based on kernels $\{\hat{\mK}^\ell\}_{\ell=1}^L$. 
However, we notice that $\{\mK^\ell\}_{\ell=1}^L$ is restricted by $\sigma_{\mK}$ class and does not cover all possible kernels. 
This issue can be addressed by defining an {\bf inducing points block} $IPB$ to replace $RB$ in Alg. \ref{alg:constructBNN}. We can show that the derived variational posterior for $\mN(\mS)$
can now be viewed as posterior approximation for a 
general DGP. 
\begin{definition}
For a kernel $\mK$, $IPB$ can be constructed by choosing $r$ additional points $\bZ$ (inducing points), taking the inputs $\bx$ and outputting an $r$-dimension vector $\mK(\bx,\bZ)\mK(\bZ,\bZ)^{-1/2}$.  
\end{definition}
\begin{theorem}\label{thm:KandhK}
Using the variational approximation \cite{Sal17} for the posterior of a DGP defined on $\{\mK^\ell\}_{\ell=1}^L$ with inducing points, we can obtain the same variational posterior $q(\{\bF^\ell\}_{\ell=1}^L)$ and evidence lower bound ELBO as the variational posterior for $\mN(\mS)$ (with $IPB$) except a constant offset that does not depend on training (see supplement). 
\end{theorem}
{\bf Summary.} We see that the main difference between random features \cite{Cut17} and 
inducing points \cite{Sal17} is that one uses $\bPhi$ and the other uses $\mK(\bx,\bZ)\mK(\bZ,\bZ)^{-1/2}$ as a rank $r$ basis to approximate the kernel.

\subsection{The BNN $\mN(\mS)$ is extremely flexible}
\begin{figure}
	\centering

\end{figure}

We have already shown that the BNN $\mN(\mS)$ from our $\mS$ can 
be seen as approximation for DGP posterior and can be trained efficiently. Now, we show that with a few small changes, interesting special cases emerge. 
To do so, the changes to Alg. \ref{alg:constructBNN}, Definition \ref{def:BNN} and the variational posterior are,

 \begin{inparaenum}[\bfseries {Change} 1)]
 \item In Alg. \ref{alg:constructBNN}, inside the inner-most loop, we have one $RB$ ($IPB$). We allow taking out $RB$ ($IPB$) entirely
   or replacing it by multiple sequential $RB$s ($IPB$s) as long as they are matched.\\
 \item Earlier, we assumed that the variational posterior $q$ for $\bv^\ell_i$ follows a normal distribution.
   We now allow it to follow a probability mass function and a mixture of two probability mass functions.\\
 \item Earlier in Def. \ref{def:BNN}, the prior $p(\bv)$ is a normal distribution. We allow other forms of priors to encourage other types of regularization,
   such as a Laplace distribution for $\ell_1$ sparsity.
  \end{inparaenum}\\

{\bf Remark.}
Change 1 allows us to include the classical BNN settings, such as the MC dropout\cite{Gal16} and deep random features concept \cite{Dan16}. 
Since the exact posterior can 
be multi-modal, we may want more flexibility beyond normal distribution to approximate it. 
Therefore, we use Change 2 where the probability mass function results in a standard NN and a
mixture of two probability mass function results in  MC dropout  \cite{Gal16}, which are 
both easy to realize in the optimization of ELBO. For Change 3, when the prior is the normal distribution, the prior for $\mN(\mS)$ can be seen as a DGP prior as we have shown. However, the normal distribution is related to the $\ell_2$ regularization from the KL divergence term in the ELBO, while we may need priors in BNNs related to the Lasso or group Lasso type penalties
to encourage sparse structure. This intuition motivates Change 3.\\

Let us see examples of previous works in our framework.
{\bf First}, consider the case where do not use any $RB$ in constructing $\mN(\mS)$.
This gives us a kernel $\mK$ with $\mK(\bx,\bx') = \sigma(\bx)^T\sigma(\bx')$ as in Fig. \ref{fig:blocks} (e). 
Further, let us use $q(\bv)$ as a mixture of two probability mass functions to approximate the posterior. 
This change leads to MC dropout \cite{Gal16}. 
{\bf Second}, let us allow multiple $RB$s in constructing $\mN(\mS)$ as in 
Fig. \ref{fig:blocks}(f). This ends up 
representing the deep random feature idea in \cite{Dan16} for GPs. 
{\bf Third}, in Fig.\ref{fig:blocks} (g), we show another construction that 
represents deep kernel learning \cite{Wil16}, where $IPB$ or $RB$ is not used at all; 
the variational posterior $q(\bv)$ is a probability mass function except the last layer. \\

{\bf Remark.} Though the framework is flexible, we note that an arbitrary construction can lead to overfitting. 
Therefore, one still needs to refer to the 
previous constructions \cite{Sal17, Cut17, Gal16, Dan16, Wil16} and using our proposal as a guide, 
consider how to generalize the construction to other $\mS$s, how to refine 
them to obtain more compact forms.\\

{\bf Computation skeleton and structure:} As we have emphasized, the computation skeleton 
captures the most important information of the structure so it helps when 
we have some structure assumptions for the BNN. In Fig. \ref{fig:blocks} (g), we see 
a computation skeleton for multi-task learning where the first layer defines a shared low level function and 
the second layer defines individual high level functions for each task. In Fig. \ref{fig:blocks} (h), we have 
an additive structure where a large neural network is composed by summing several sub neural networks. 
For those computation skeletons $S$s, 
the construction process for BNN $\mN(\mS)$s directly comes from Alg. \ref{alg:constructBNN} and Definition \ref{def:BNN}.

\section{Statistical inference through AddNN: Bayesian additive neural network}\label{sec:interaction}
Let us see an example of using additive structure in BNNs to detect interactions for interpretability.
We specialize our framework to define a Bayesian {\em additive} neural network.
In statistics, given a function $f^\ast$ between inputs $\bx = (x_1,...,x_p)$ and output $y$, one can define interaction $\I_T$ over a subset of inputs $T$  through ANOVA decomposition $\I_T(\bx_T) = \prod_{i\in T} (I_{x_i}-\mE_{x_i})\prod_{j\notin T}\mE_{x_j}f^\ast(x_1,...,x_p)$. 
We can design an additive neural network (AddNN) to partially represent the ANOVA decomposition: $f(\bx)= \sum_{j=1}^k g_j(\bx)$, where every $g_j$ is a NN with the first layer regularized by group Lasso type penalty. We can use post-training ANOVA decomposition (replace $\mE$ in ANOVA by the empirical expectation $\mE^n$ with $n$ samples) to measure the interactions:

\begin{equation}\label{eq:empANOVA}
 \I^n_T(\bx_T) = \prod_{i\in T} (I_{x_i}-\mE^n_{x_i})\prod_{j\notin T}\mE^n_{x_j}f(x_1,...,x_p).
\end{equation}

We have the following theorem for the complexity of calculating this measure for AddNN,

\begin{theorem}
  If there exist inputs clusters $\{T^\ast_j\}_{j=1}^{k^\ast}$ such that $f^\ast(\bx) = \sum_{j=1}^{k^\ast}g^\ast_j(\bx_{T^\ast_j})$ with $k^\ast$
  of the order of a polynomial in $p$ and $c = \max_{j=1}^{k^\ast}|T^\ast_j| = O(\log p)$, then there exists a trained AddNN that predicts $\by$ well and restricts the number
  of possible interactions to be at most a polynomial in $p$.
Further, if every sub neural network has $L$ layers with $d$ hidden units, then the complexity of \eqref{eq:empANOVA} is at most $n^ck^\ast d^{2L-1}$, which is also polynomial in $p$. 
\end{theorem}

The result does not hold for an arbitrary NN, which has $2^p$ possible interactions. The complexity of \eqref{eq:empANOVA} is $n^p(k^\ast d)^{2L-1}$, exponential in $p$. 
AddNN is far more efficient when $f^\ast$ has additive structure.\\

{\bf Uncertainty in AddNN:} We show the additive neural network computation skeleton as an example in Fig. \ref{fig:blocks} (h). Then, we can easily construct the Bayesian formulation of AddNN with various uncertainty estimates methods. We only require a specific design of the first layer for variable selection. For every sub-neural network, the first layer is
only built with $FB$ and the prior on the weights is $p(\bV^1)\sim \exp(-\sum_{i=1}^p||\bv^1_i||_2)$ where $\bv^1_i$ refers to the weight vector emanating
from the $i$th input. The variational posterior $q(\bV^1)$
is the probability mass function to make top layers stable. 

\section{Experiments}\label{sec:exp}
We first evaluate the performance of our AddNN model
on synthetic experiments
for regression and interaction detection for additive functions.
Then, we use Alg. \ref{alg:constructBNN} to
construct four different types of AddNNs (where each provides uncertainty estimates)
and check its utility for prediction and identifying interaction strength.
Finally, we show how AddNN can infer
main effects and statistical interactions of features with
uncertainties for interpretability. Further,
we use AddNN on eight benchmark datasets used in existing papers
to show that our model offers competitive results.

\paragraph{AddNN, BNN, BART and NID on prediction accuracy and interaction detection.}
We compare AddNN with BNN (with a single neural network),
BART (Bayesian additive regression tree) and NID (Neural interaction detection) in terms of
prediction accuracy and interaction detection.
For AddNN, we use the setup in \S\ref{sec:interaction}, where the group Lasso penalty is applied on the first layer.
We use $10$ compact sub-NNs for AddNN and a
single (but more complex) neural network for BNN (see supplement).
For BART and NID, we use the setup in \cite{Chi10, Tsa17}.
Both AddNN and BNN here are based on the MC dropout type construction (see \S\ref{sec:skeleton}).
First, we compare RMSE (root mean-squared-error). We run $4$ synthetic
experiments using the functions in the left Tab. \ref{tab:realdata}.
For every experiment, we use one function $f$ in the left Tab. \ref{tab:realdata} to
generate $5000$ train/test samples ($10$ features, $1$ response),
where for every input $\bx$, each dimension of the inputs are i.i.d. generated from the uniform distribution
on $(0,1]$ and the response $y$ is $y =f(\bx)+\epsilon$, with $\epsilon\sim N(0,1)$.
  From Tab. \ref{tab:accuracy}, we see that the AddNN yields
  comparable (and sometimes better) RMSE compared to baselines. 
  Though the prediction performance is similar, note that AddNN is a much
  more compact design: AddNN has just $\sim$500 edges while the BNN has $7000$ edges, NID has $2000$ edges and
  BART has $200$ trees (see Tab. \ref{tab:accuracy}).

  Next, we compare AddNN and NID for interaction detection (other two baselines are not applicable).
  To detect interactions, AddNN first calculates the interaction functions from \eqref{eq:empANOVA},
  then their empirical $\ell_2$ norms are used as the ``interaction strength'', and then AddNN selects the top $k$ interactions.
  Possible interaction candidates are based on the group-Lasso clusters for every ``sub-NN''
  in our additive model.
  For NID, we use the setup in \cite{Tsa17}.
  We run the same experiments as the RMSE setting using left Tab. \ref{tab:realdata}.
  To assess ranking quality, we use the top-rank recall metric \cite{Tsa17}: 
  a recall of interaction rankings where only those interactions
  that are correctly ranked before we encounter any false positives are considered.
  Only one superset interaction from each sub-function of $f$ is counted as a true interaction.
  From Tab. \ref{tab:accuracy}, we see that the AddNN outperforms NID for interaction detection. \\

\begin{table}
	\caption{\label{tab:realdata} (left) Synthetic functions used in our experiments, based on \cite{Tsa17};(right) Average test performance in RMSE for AddNN(ours) and MC dropout on benchmarks. }
	\begin{minipage}{.5\textwidth}
		\centering
		{
			\begin{tabular}{cc}
				Method                 & Formula                                                                                   \\ \hline
				\multirow{2}{*}{$f_1$} & $10\sin(\pi x_1x_2) + 20(x_3 - .5)^2$                                                     \\
				& $+ 10x_4 + 5x_5$                                                                          \\
				\hline
				\multirow{2}{*}{$f_2$} & $10\exp(x_1x_2) - 20\cos(x_3+x_4+x_5)$                                                    \\
				& $+ 7\arcsin(x_9x_{10})$                                                                   \\ \hline
				\multirow{2}{*}{$f_3$} & $\exp(|x_1x_2| + 1) + \exp(|x_3+x_4| + 1)$                                                \\
				& $- 19\cos(x_5 + x_6) - 10\sqrt{x_8^2 + x_9^2 + x_{10}^2}$ \\
				\hline
				\multirow{2}{*}{$f_4$} & $\frac{1}{1+x_1^2 + x_2^2 + x_3^2} - 5\sqrt{\exp(x_4 + x_5)}$                             \\
				& $+ 10|x_6 + x_7| + 6x_8x_9x_{10}$
				\\ \hline                                                        
			\end{tabular}
		}
	\end{minipage}
	\begin{minipage}{.5\textwidth}
		\centering
		{
			\begin{tabular}{|c|c|c|}
				\hline
				Measure & RMSE (AddNN) & RMSE (\cite{Gal16})\\
				\hline
				Boston & $3.03\pm 0.12$ & $2.97 \pm 0.19$ \\
				\hline
				Concrete & $5.18 \pm 0.14$ & $5.23 \pm 0.12$ \\
				\hline
				Energy & $0.65 \pm 0.03$ & $1.66 \pm 0.04$ \\
				\hline
				Kin8nm & $0.07 \pm 0.00$ & $0.10 \pm 0.00$ \\
				\hline
				Naval & $0.01 \pm 0.00$ & $0.01 \pm 0.00$\\
				\hline
				Power & $4.04\pm 0.03$ & $4.02 \pm 0.04$\\
				\hline
				Protein & $4.07\pm 0.01$ & $4.36\pm 0.01$\\
				\hline
				Wine & $0.66\pm 0.01$ & $0.62 \pm 0.01$\\
				\hline
			\end{tabular}
		}
	\end{minipage}
\end{table}
\begin{table}
  \caption{\label{tab:accuracy}
    Comparisons between AddNN, BNN, BART and NID. BNN and BART do not detect interactions.}
\centering
{
\begin{tabular}{|c|c|c|c|c|c|c|c|c|c|c|}
\hline
 & \multicolumn{4}{|c|}{RMSE} & \multicolumn{6}{|c|}{Top rank recall (noise, $\sigma^2=1,3,5$)}\\
\hline
 & AddNN ($0.5k$) & BNN ($7k$) & BART & NID ($20k$) & \multicolumn{3}{|c|}{\textbf{Ours} (AddNN)} & \multicolumn{3}{|c|}{NID}\\
\hline
$f_1$ & $1.07 \pm 0.01$ & $1.15 \pm 0.01 $ & $1.07 \pm 0.01$ & $1.09 \pm 0.01$ & \textbf{1} & \textbf{1} & \textbf{1} & 1 & 1 & 1\\
\hline
$f_2$ & $1.16 \pm 0.01$ & $1.22 \pm 0.02 $ & $1.43 \pm 0.02$ & $1.44 \pm 0.02$ & \textbf{1} & \textbf{1} & \textbf{2/3} & 1 & 2/3 & 0\\
\hline
$f_3$ & $1.35 \pm 0.01$ & $1.32 \pm 0.01 $ & $1.24 \pm 0.02$ & $1.42 \pm 0.02$ & 3/4 & \textbf{3/4} & \textbf{2/4} & 1 & 2/4 & 1/4\\
\hline
$f_4$ & $1.13 \pm 0.01$ & $1.13 \pm 0.01 $ & $1.17 \pm 0.01$ & $1.40 \pm 0.02$ & \textbf{3/4} & \textbf{3/4} & \textbf{2/4} & 2/4 & 2/4 & 1/4\\
\hline
\end{tabular}
}
\end{table}

{\bf Four different types of AddNN.} As described in \S\ref{sec:skeleton}, we can derive other uncertainty schemes
\begin{wrapfigure}[12]{r}{0pt}
	\includegraphics[width=0.35\textwidth]{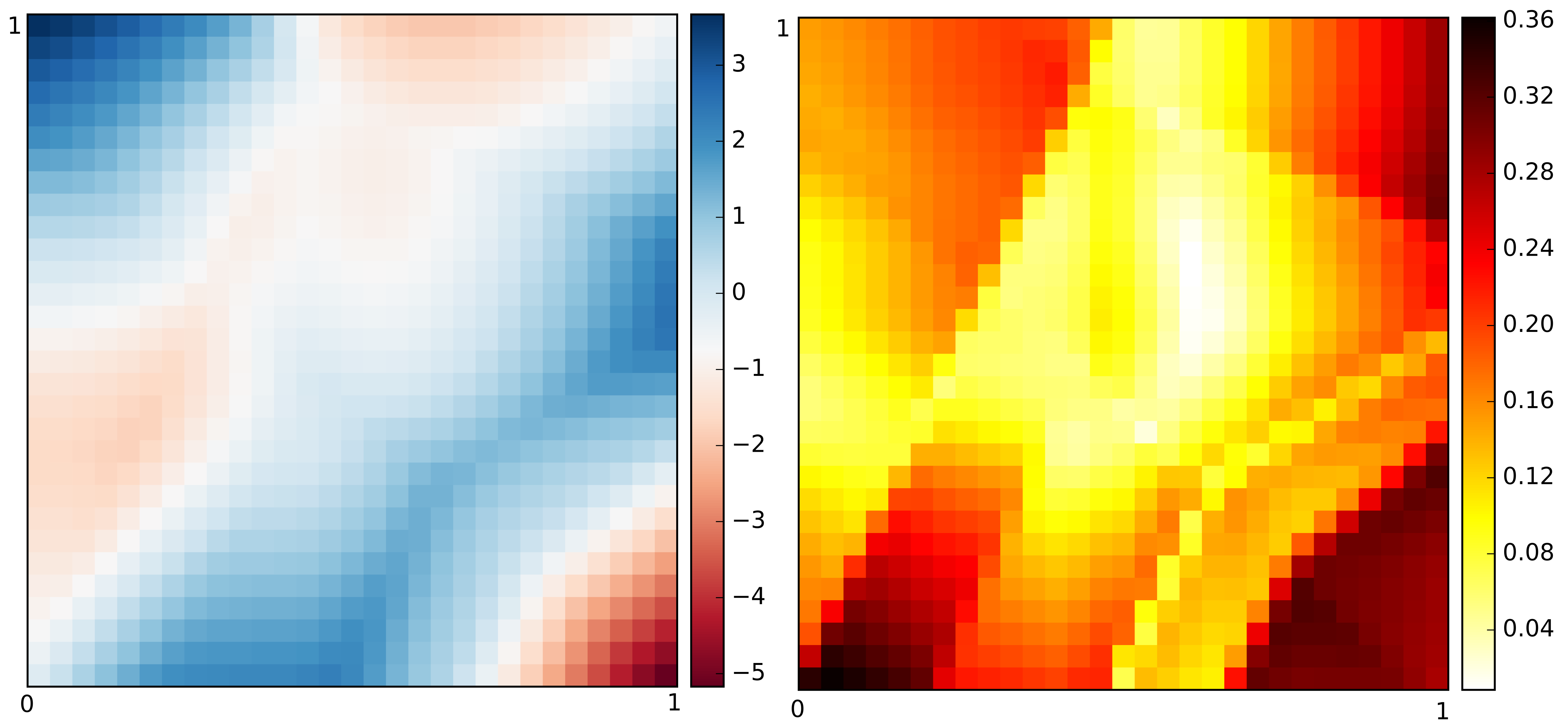}
	\caption{\small\label{fig:Finteract} Interaction between $x_1$ and $x_2$ for $f_1$ in Tab. \ref{tab:realdata}.
          Left (and right) image shows the mean interaction (and its standard deviation). 
	}		
\end{wrapfigure}
using AddNN.
Then, we can calculate uncertainty based on each of these schemes.
Here, we use the mean log likelihood (MLL) and the empirical $\ell_2$ norm of the
interaction or main effect function from \eqref{eq:empANOVA}.
The empirical $\ell_2$ norm measures the strength of the interaction or main effect.
We calculate the uncertainty for AddNN but do not compare with NID \cite{Tsa17}
since it cannot model uncertainty.
Here, we show an example for $f_1$ from left Tab. \ref{tab:realdata}.
Tab. \ref{tab:uncertainty} shows that that all four methods (derivable from our proposal)
correctly yields interactions between $x_1$ and $x_2$ as well as the main effects.
AddNN provides predictions {\em and} interactions with uncertainties. 

{\bf Uncertainty and interpretability using AddNN.}
We show one representative example showing how \eqref{eq:empANOVA} with our formulation can be used
to model the interaction between $x_1$ and $x_2$ for
$f_1$ in Tab. \ref{tab:realdata}. We plot the average interaction function and the uncertainty function in
Fig. \ref{fig:Finteract} as a heatmap -- this ability is rarely available
for deep neural network models and can be very useful for interpretability. \\


{\bf Benchmark experiments.}
Finally, we apply AddNN on common datasets used by other authors \cite{Gal16, Sal17}. Here, as shown in right Tab. \ref{tab:realdata}, we find that AddNN (which is a more compact model) yields competitive performance in addition to the other features it natively
provides such as interaction (interpretability) and
uncertainty discussed above. This implies that these additional benefits do not come at a cost of
performance. The complete table with previous works are in supplement. 

\begin{table}
  \caption{\label{tab:uncertainty} Constructing multiple types of uncertainty estimates for AddNN of $f_1$.}
\centering
\begin{adjustbox}{max width=\textwidth}
{
\begin{tabular}{|c|c|c|c|c|c|c|c|}
\hline
\multirow{2}{*}{Measure} &  \multirow{2}{*}{MLL} &
\multicolumn{1}{|c|}{Interaction} & \multicolumn{5}{|c|}{Main effect}\\
\cline{3-8}
\multirow{2}{*}{} & \multirow{2}{*}{} & $(1,2)$ & $1$ & $2$ & $3$ & $4$ & $5$\\
\hline
MC dropout & $-1.61 \pm 0.09$ & $1.51 \pm 0.05$ & $2.44\pm 0.15$ & $2.35\pm 0.10$ & $1.69 \pm 0.06$ & $3.18 \pm 0.04$ & $1.63\pm 0.03$\\
\hline
RF & $-1.60 \pm 0.09$ & $1.52 \pm 0.06$ & $2.39 \pm 0.08$ & $2.31 \pm 0.11$ & $1.70 \pm 0.11$ & $3.19\pm 0.06$ & $1.61\pm 0.04$\\
\hline
DKL & $-1.53 \pm 0.08$ & $1.59 \pm 0.04$ & $2.44\pm 0.23$ & $2.32 \pm 0.13$ & $1.70\pm 0.08$ & $3.16 \pm 0.06$ & $1.61 \pm 0.04$\\
\hline
DRF & $-1.56 \pm 0.07$ & $1.40 \pm 0.02$ & $2.36 \pm 0.05$ & $2.35 \pm 0.10$  & $1.71 \pm 0.03$ & $3.15 \pm 0.03$ & $1.59 \pm 0.02$\\
\hline
\end{tabular}
}
\end{adjustbox}
\end{table}
\vspace{-7pt}
\section{Discussion}
We presented a scheme by adapting the computation skeleton idea to construct BNNs. Our models can be trained using modern innovations
including mini-batch training, dropout, and automatic differentiation. We showed that a broad class of BNNs, realized by our framework,
ties nicely to DGPs, deep kernel learning, MC dropout and other topics.
As a special case,  we proposed an Bayesian additive neural network
that competes favorably with state-of-the-arts methods and provides uncertainty and interpretability, via statistical interactions.\\

{\bf Acknowledgments} We thank Grace Wahba and Ronak Mehta for suggestions and ideas. This work is supported by NIH Grants R01 AG040396, R01 EB022883, UW CPCP AI117924, and R01 AG021155, and NSF Awards DMS 1308877 and CAREER 1252725. We received partial support from NIH Grants UW ADRC AG033514 and UW ICTR 1UL1RR025011 and funding from a University of Wisconsin–Madison/German Center for Neurodegenerative Diseases collaboration initiative.

\bibliographystyle{plain}
\bibliography{BANN}

\newpage

\section{Supplement}
\setcounter{theorem}{0}
\setcounter{definition}{0}
In this supplement, we first discuss the extensions of the model to multiple outputs and classification and discuss how to incorporate the bias terms in NN. Then we show the proofs for the theorems in the main body. Finally, we present details for models used in experiments and provide more experiment results.

\subsection{The extension to multiple outputs, classification and including bias terms}
For the output $\by\in \mR^d$, we permit our computation skeleton to have $d$ output nodes as well. Then after we run our construction algorithm, we obtain a BNN with $d$ outputs $\bdf^L$ at the last layer. Then the analysis and properties for the single output case also hold for the multiple output case.\\

In regression task, we assume the likelihood $p(\by|\bF^L)$ of the output $\by$ to be a normal distribution, given the input matrix $\bX$ with $n$ samples and the relevant output $\bF^L$ at the last layer of BNN. We output $\bF^L$ to estimate the mean of $\by$. The relevant loss in the optimization of ELBO is the mean square loss. This is usually considered for the regression task. In a classification task with $\by \in \{0,1,...,k\}$ in $k$ categories, we assume the likelihood $p(y_i|\bF^L_{i.})= \frac{\exp(\bF^L_{iy_i})}{\sum_{j=1}^k \exp(\bF^L_{ij})}$ for $1\leq i\leq n$. Then the relevant loss in the optimization of ELBO is 

$$\frac{1}{n}\sum_{i=1}^n \log\left(\frac{\exp(\bF^L_{iy_i})}{\sum_{j=1}^k \exp(\bF^L_{ij})}\right).$$ 

We can add bias terms into the framework. The bias term with random weight can either be incorporated into the construction of $RB$ \cite{Dan16} or be treated as a parameter in the BNN \cite{Gal16}. We refer one to see these two works for incorporating bias terms.\\

{\bf Remark.} The statements for the extension to multiple outputs and classification hold for DGPs as well.

\subsection{Proofs for theorems in the main body}
In this section, we give the proofs for theorems in the main body.
\subsubsection{The proof for the relation between activation functions and kernels}
First, we prove theorems on the relation between activation functions and kernels.

{\bf The uniform concentration bound for $C$-bounded activation functions and its proof}

First, we present the uniform concentration bound for $C$-bounded activation functions and its proof. 

{\bf Theorem 0.} 
If the activation function $\sigma_{\mK}$ is $C$-bounded, meaning it is continuously differentiable and $||\sigma_{\mK}||_\infty, ||\sigma_{\mK}'||_\infty \leq C$, 
then for every $1<\ell\leq L$, on a compact set $\mM\in \mR^d$ with diameter $diam(\mM)$, 
with probability at least $1 - c_1 diam(\mM)^2\exp\left\{-\frac{\epsilon^2 r}{8(1+d)C}\right\}$,

$$\sup_{\sigma(\bdf^{\ell-1}(\bx)), \sigma(\bdf^{\ell-1}(\bx'))\in \mM}
|\hat{\mK}^\ell(\bdf^{\ell-1}(\bx),\bdf^{\ell-1}(\bx'))-\mK^\ell(\bdf^{\ell-1}(\bx),\bdf^{\ell-1}(\bx'))| \leq \epsilon,$$

for a constant $c_1>0$.

\begin{proof}
(a) For a $C$-bounded activation function, since $||\sigma_{\mK}(\cdot)||_\infty\leq C$, for fixed $\bdf^{\ell-1}(\bx)$ and $\bdf^{\ell-1}(\bx')$, 
the $r$ random variables $\{\sigma_{\mK}(\sigma(\bdf^{\ell-1}(\bx))^T\bw_i)\sigma_{\mK}(\sigma(\bdf^{\ell-1}(\bx'))^T\bw_i)\}_{i=1}^r$ are independent and lie in a bounded interval $[-C,C]$. 
Then using Hoeffdings' inequality, we get that

\begin{equation}\label{eq: thm1proof1}
\mP(|\hat{\mK}^\ell(\bdf^{\ell-1}(\bx),\bdf^{\ell-1}(\bx'))-\mK^\ell(\bdf^{\ell-1}(\bx),\bdf^{\ell-1}(\bx'))| \geq \epsilon)\leq 2\exp(-\frac{2r\epsilon^2}{4C^2}).
\end{equation}

Next we show that for a compact set $\mM$ of $\mR^d$ with diameter $diam(\mM)$, with probability at least $1-2^{11}\left(\frac{C^4 d\ diam(\mM)^2}{\epsilon^2 r}\right)^{\frac{d}{1+d}}\exp\left\{-\frac{r\epsilon^2}{8(1+d)C^2}\right\}$,

$$\sup_{\sigma(\bdf^{\ell-1}(\bx)), \sigma(\bdf^{\ell-1}(\bx'))\in \mM}
|\hat{\mK}^\ell(\bdf^{\ell-1}(\bx),\bdf^{\ell-1}(\bx'))-\mK^\ell(\bdf^{\ell-1}(\bx),\bdf^{\ell-1}(\bx'))| \leq \epsilon .$$

Since $\mM$ has diameter $diam(\mM)$, we can find $\delta$-net that covers $\mM$ using at most $T = (4 diam(\mM)/\delta)^d$ balls of radius $\delta$. 
Let $\{\Delta_i\}_{i=1}^{T}$ denote the centers of these balls. Then using \eqref{eq: thm1proof1} and union bounds, for any two centers, such as $\Delta_1$ and $\Delta_2$, with probability at least $1-2\exp(\log(T^2)-\frac{2r\epsilon^2}{16C^2})$,

\begin{equation}\label{eq:thm1proof2}
|\hat{\mK}^\ell(\Delta_1,\Delta_2)-\mK^\ell(\Delta_1,\Delta_2)| \leq \frac{\epsilon}{2}.
\end{equation}

For the function $\bu(\bx,\bx')=\hat{\mK}^\ell(\bdf^{\ell-1}(\bx),\bdf^{\ell-1}(\bx'))-\mK^\ell(\bdf^{\ell-1}(\bx),\bdf^{\ell-1}(\bx'))$, we have the inequality from partial derivative that

\begin{equation}\label{eq:thm1proof3}
|\bu(\bx,\bx')-\bu(\bx_0,\bx'_0)| \leq L_{\sigma_{\mK}}(||\sigma(\bdf^{\ell-1}(\bx))-\sigma(\bdf^{\ell-1}(\bx_0))||_2+||\sigma(\bdf^{\ell-1}(\bx'))-\sigma(\bdf^{\ell-1}(\bx'_0))||_2),
\end{equation}

where 

\begin{align*}
L_{\sigma_{\mK}} &= \arg\max_{\sigma(\bdf^{\ell-1}(\bx)),\sigma(\bdf^{\ell-1}(\bx'))\in \mM}
||\frac{1}{r}\sum_{i=1}^r \frac{\partial \sigma_{\mK}(\sigma(\bdf^{\ell-1}(\bx))^T\bw_i)}{\partial \sigma(\bdf^{\ell-1}(\bx))}\sigma_{\mK}(\sigma(\bdf^{\ell-1}(\bx'))^T\bw_i)\\
&-\mE_\bw \frac{\partial \sigma_{\mK}(\sigma(\bdf^{\ell-1}(\bx))^T\bw)}{\partial \sigma(\bdf^{\ell-1}(\bx))}\sigma_{\mK}(\sigma(\bdf^{\ell-1}(\bx'))^T\bw)||_2\\
&= 
||\frac{1}{r}\sum_{i=1}^r \frac{\partial \sigma_{\mK}(\sigma(\bdf^{\ell-1}(\bx^\ast))^T\bw_i)}{\partial \sigma(\bdf^{\ell-1}(\bx))}\sigma_{\mK}(\sigma(\bdf^{\ell-1}(\bx'^\ast))^T\bw_i)\\
&-\mE_\bw \frac{\partial \sigma_{\mK}(\sigma(\bdf^{\ell-1}(\bx^\ast))^T\bw)}{\partial \sigma(\bdf^{\ell-1}(\bx))}\sigma_{\mK}(\sigma(\bdf^{\ell-1}(\bx'^\ast))^T\bw)||_2.
\end{align*}

We also have that

\begin{align*}
\mE L_{\sigma_{\mK}}^2 &= 
\mE||\frac{1}{r}\sum_{i=1}^r \frac{\partial \sigma_{\mK}(\sigma(\bdf^{\ell-1}(\bx^\ast))^T\bw_i)}{\partial \sigma(\bdf^{\ell-1}(\bx))}\sigma_{\mK}(\sigma(\bdf^{\ell-1}(\bx'^\ast))^T\bw_i)\\
&-\mE_\bw \frac{\partial \sigma_{\mK}(\sigma(\bdf^{\ell-1}(\bx^\ast))^T\bw)}{\partial \sigma(\bdf^{\ell-1}(\bx))}\sigma_{\mK}(\sigma(\bdf^{\ell-1}(\bx'^\ast))^T\bw)||_2^2\\
&=
\mE||\frac{1}{r}\sum_{i=1}^r \frac{\partial \sigma_{\mK}(\sigma(\bdf^{\ell-1}(\bx^\ast))^T\bw_i)}{\partial \sigma(\bdf^{\ell-1}(\bx))}\sigma_{\mK}(\sigma(\bdf^{\ell-1}(\bx'^\ast))^T\bw_i)||_2^2\\
&-||\mE_\bw \frac{\partial \sigma_{\mK}(\sigma(\bdf^{\ell-1}(\bx^\ast))^T\bw)}{\partial \sigma(\bdf^{\ell-1}(\bx))}\sigma_{\mK}(\sigma(\bdf^{\ell-1}(\bx'^\ast))^T\bw)||_2^2\\
&=
\frac{1}{r^2}\sum_{i=1}^r\mE|| \frac{\partial \sigma_{\mK}(\sigma(\bdf^{\ell-1}(\bx^\ast))^T\bw_i)}{\partial \sigma(\bdf^{\ell-1}(\bx))}\sigma_{\mK}(\sigma(\bdf^{\ell-1}(\bx'^\ast))^T\bw_i)||_2^2\\
&-\frac{1}{r^2}||\mE_\bw \frac{\partial \sigma_{\mK}(\sigma(\bdf^{\ell-1}(\bx^\ast))^T\bw)}{\partial \sigma(\bdf^{\ell-1}(\bx))}\sigma_{\mK}(\sigma(\bdf^{\ell-1}(\bx'^\ast))^T\bw)||_2^2\\
&\leq 
\frac{1}{r^2}C^4 \sum_{i=1}^r \mE||\bw_i||_2^2\\
&=
\frac{C^4 d}{r}.
\end{align*}

Therefore, by Markov's inequality,

$$\mP(L_{\sigma_{\mK}} \geq \frac{\epsilon}{4 \delta}) \leq \mE L_{\sigma_{\mK}}^2 \frac{16\delta^2}{\epsilon^2} \leq \frac{16 \delta^2 C^4 d}{\epsilon^2 r}$$

Then using Eq. \eqref{eq:thm1proof3}, with probability at least $1-\frac{16 \delta^2 C^4 d}{\epsilon^2 r}$,

$$|u(\bx,\bx')-u(\bx_0,\bx'_0)|\leq \frac{\epsilon}{2}$$

This inequality combined with Eq. \eqref{eq:thm1proof2} enables us to conclude that

$$\sup_{\sigma(\bdf^{\ell-1}(\bx)), \sigma(\bdf^{\ell-1}(\bx'))\in \mM}
|\hat{\mK}^\ell(\bdf^{\ell-1}(\bx),\bdf^{\ell-1}(\bx'))-\mK^\ell(\bdf^{\ell-1}(\bx),\bdf^{\ell-1}(\bx'))| \leq \epsilon .$$

with probability at least $1-\frac{16 \delta^2 C^4 d}{\epsilon^2 r}-2\exp(\log(T^2)-\frac{2r\epsilon^2}{16C^2})$. Recall that $T = (4 diam(\mM)/\delta)^d$, so the probability has a format of $1-\kappa_1 \delta^2 - \kappa_2 \delta^{-2d}$ for $\delta$. By setting $\delta = \frac{\kappa_2}{\kappa_1}^{\frac{1}{2+2d}}$, we have the probability as $1-2\kappa_1^{\frac{2d}{2+2d}}\kappa_2^{\frac{2}{2+2d}}$. So the probability is at least

$$1-2^{11}\left(\frac{C^4 d\ diam(\mM)^2}{\epsilon^2 r}\right)^{\frac{d}{1+d}}\exp\left\{-\frac{r\epsilon^2}{8(1+d)C^2}\right\}$$

\end{proof}

{\bf Proof of Theorem 1 for ReLU}

We have seen how to control the distance between empirical kernel and the expectation kernel uniformly for $C$-bounded activation functions, now we present the proof for ReLU activation functions.

\begin{theorem}\label{thm:1}
If the activation function $\sigma_{\mK}$ is ReLU, then for every $1\leq \ell \leq L$, on a compact set $\mM \in \mR^d$ with diameter $diam(\mM)$ and $\max_{\Delta\in \mM}||\Delta||_2\leq c_\mM$, with probability at least $1-c_1c_{\mM} diam(\mM)^2\exp\left\{-\frac{r\epsilon^2}{8(1+d)\nu_{\mM}^2}\right\}$,

$$\sup_{\sigma(\bdf^{\ell-1}(\bx)), \sigma(\bdf^{\ell-1}(\bx'))\in \tilde{\mM}}
|\hat{\mK}^\ell(\bdf^{\ell-1}(\bx),\bdf^{\ell-1}(\bx'))-\mK^\ell(\bdf^{\ell-1}(\bx),\bdf^{\ell-1}(\bx'))| \leq \epsilon,$$

for a constant $c_1> 0$ and a parameter $\nu_\mM$ depending on $\mM$. Here, $\tilde{\mM}$ 
specifies that we require $\sigma(\bdf^{\ell-1}(\bx))$ and $\sigma(\bdf^{\ell-1}(\bx'))$ to be two vectors in $\mM$ 
that are not collinear.
\end{theorem}

\begin{proof}
For the ReLU activation $\sigma_{\mK}(x) = \max(0, x)$, we use concentration bound for sub-exponential random variable to show the result for fixed points. 
We define $u = \sigma_{\mK}(\sigma(\bdf^{\ell-1}(\bx))^T\bw)\sigma_{\mK}(\sigma(\bdf^{\ell-1}(\bx'))^T\bw)$. 
Our first goal is to compute $\mE_\bw[e^{\lambda\bu}]$. 
Since $\bw$ follows a normal distribution that is symmetric, how we choose axis does not influence the results. 
Therefore, we choose axis such that $\sigma(\bdf^{\ell-1}(\bx)) = \be_1||\sigma(\bdf^{\ell-1}(\bx))||_2$ and $\sigma(\bdf^{\ell-1}(\bx')) = (\be_1\cos\theta + \be_2\sin\theta) ||\sigma(\bdf^{\ell-1}(\bx'))||_2$ where $\be_1$ and $\be_2$ refer to standard vector for the first and second axis. 
We denote $C_f = ||\sigma(\bdf^{\ell-1}(\bx))||_2 ||\sigma(\bdf^{\ell-1}(\bx'))||_2\leq c_{\mM}^2$. 

\begin{align*}
\mE_\bw[e^{\lambda u}] &= \frac{1}{2\pi}\int_{-\infty}^\infty dw_1\int_{-\infty}^\infty dw_2 e^{-\frac{1}{2}(w_1^2+w_2^2)}e^{\lambda C_f \max(0,w_1)\max(0,w_1\cos\theta+w_2\sin\theta)}.
\end{align*}

We switch $(w_1, w_2)$ by $(\tw_1, \tw_2)= (w_1, w_1\cos\theta + w_2\sin\theta)$, then we get that

\begin{align*}
\mE_\bw[e^{\lambda u}]& =  \frac{1}{2\pi \sin\theta}\int_0^\infty d\tw_1\int_0^\infty d\tw_2 e^{-\frac{\tw_1^2 + \tw_2^2 -2\tw_1\tw_2\cos\theta }{2\sin\theta^2}}e^{\lambda C_f \tw_1\tw_2}.
\end{align*}

We switch $(\tw_1, \tw_2)$ by $(\tr, \tphi)$ with $\tw_1 = \tr\sin\tphi$ and $\tw_2 = \tr\cos\tphi$, then we get that

\begin{align*}
\mE_\bw[e^{\lambda u}] & = \frac{1}{2\pi \sin\theta}\int_0^{\pi/2}d\tphi\int_0^\infty \tr d\tr e^{-\tr^2\frac{ 1-\sin 2\tphi\cos\theta }{2\sin\theta^2}}e^{\tr^2 \frac{\lambda C_f\sin 2\tphi}{2}}.
\end{align*}

Through the known mean calculation of half normal distribution that

$$\frac{a\sqrt{2}}{\sqrt{\pi}} = \int_{x\geq 0} x dx \frac{\sqrt{2}}{a\sqrt{\pi}}\exp(-\frac{x^2}{2a^2}),$$

for any $a$, we know that

$$a^2 = \int_{x\geq 0} x dx \exp(-\frac{x^2}{2a^2}),$$

for any $a$. We use this relation to calculate the integral of $\tr$ and we get that

\begin{align*}
\mE_\bw[e^{\lambda u}] &=\frac{1}{2\pi \sin\theta}\int_0^{\pi/2}d\tphi\frac{\sin\theta^2}{1-\sin 2\tphi\cos\theta - \lambda C_f \sin 2\tphi\sin\theta^2}\\
&=\frac{\sin\theta}{2\pi}\int_0^{\pi/2}d\alpha\frac{1}{1-\cos\alpha\cos\theta - \lambda C_f \cos \alpha\sin\theta^2},
\end{align*}

by switching $\tphi$ to $\alpha = 2\tphi-\frac{\pi}{2}$. It is known that

$$\int_0^\xi d\alpha \frac{1}{1-\cos\alpha\cos\theta} = \frac{1}{\sin\theta}\tan^{-1}\left(\frac{\sin\theta\sin\xi}{\cos\xi-\cos\theta}\right),$$

which can be verified by calculating the derivative of the right side \cite{Cho09}. Therefore, by setting $\xi=\frac{\pi}{2}$,

$$\int_0^{\pi}{2} d\alpha \frac{1}{1-\cos\alpha\cos\theta} =\frac{\pi-\theta}{\sin\theta}.$$

We define $\gamma = arccos(\cos\theta + \lambda C_f\sin\theta^2)$ for $0\leq \gamma \leq \pi$ under the requirement that

$$-\frac{1+\cos\theta}{C_f\sin\theta^2}\leq \lambda \leq \frac{1-\cos\theta}{C_f\sin\theta^2}.$$

Then we get that

$$\mE_\bw[e^{\lambda u}]=\frac{\sin\theta}{2\pi}\frac{\pi-\gamma}{\sin \gamma}.$$

Since $0\leq \gamma, \theta \leq \pi$, now we further assume that $-\frac{b}{C\sin\theta^2}\leq \lambda \leq \frac{b}{C\sin\theta^2}$, then $\cos\theta -b \leq \cos\gamma \leq \cos\theta +b$. For a enough small $b$, we have that
\begin{equation}\label{eq:thm1Bproof1}
\mE_\bw[e^{\lambda u}] \leq \frac{3(\pi-\theta)}{4\pi}.
\end{equation}

From \cite{Cho09},

\begin{align*}
\mE_\bw[\lambda u] =\frac{2\lambda C_f}{\pi}(\sin\theta+\cos\theta(\pi-\theta)).
\end{align*}

Therefore, we combine it with Eq. \eqref{eq:thm1Bproof1} to get that

\begin{align*}
\mE_\bw[e^{\lambda(u-\mE_\bw[u])}] \leq \frac{3(\pi-\theta)}{4\pi}\exp(-\lambda\frac{2 C_f (\sin\theta+\cos\theta(\pi-\theta))}{\pi}),
\end{align*}

for $-\frac{b}{C\sin\theta^2}\leq \lambda \leq \frac{b}{C\sin\theta^2}$ with a enough small $b$.

Because for $\frac{\pi}{2} \leq \theta \leq \pi$, $\cos\theta(\pi-\theta )\geq \cos\theta\tan(\pi-\theta)=-\sin\theta \geq 0$, we always can define $c = \frac{2 C_f (\sin\theta+\cos\theta(\pi-\theta))}{\pi}\geq 0$ and it monotonically decreases to zero at $\theta = \pi$. Then we define $\nu^2 = \frac{c^2}{2\log(\frac{4\pi}{3(\pi-\theta)})}$ that can guarantee

\begin{align*}
\mE[e^{\lambda(u-\mE[u])}]&\leq \exp(\frac{\nu^2\lambda^2}{2}),
\end{align*}

for $|\lambda|\leq \frac{b}{C\sin\theta^2}$. This means that $u$ follows a sub-exponential distribution and now we can use concentration bound to derive that

\begin{equation*}
\mP[|u - \mE[u]| \geq \epsilon]\leq 
\Big\{\begin{array}{c}
2e^{-\frac{\epsilon^2}{2\nu^2}}\ if\ 0\leq \epsilon\leq \frac{\nu^2 b}{C\sin\theta^2}\\
2e^{-\frac{\epsilon b}{2C\sin\theta^2}}\ for\ \epsilon>\frac{\nu^2 b}{C\sin\theta^2}
\end{array}.
\end{equation*}

Therefore, for a small error $\epsilon$, we can consider that 

$$\mP[|u-\mE[u]|\geq \epsilon]\leq 2e^{-\frac{\epsilon^2}{2\nu^2}}.$$

The concentration inequality also applied to the average of $r$ independent random variable $u_i$, which are defined for $r$ independent $\bw_i$. It shows that

\begin{equation}\label{eq:thm1Bproof2}
\mP[|\frac{1}{r}\sum_{i=1}^r u_i - \mE[u]| \geq \epsilon]\leq  2e^{-\frac{r\epsilon^2}{2\nu^2}}.
\end{equation}

Therefore, we obtain the concentration bound for fixed $\sigma(\bdf^{\ell-1}(\bx))$ and $\sigma(\bdf^{\ell-1}(\bx'))$. 

Next, we show the result for a set $\mM$. When $||\sigma(\bdf^{\ell-1}(\bx))||_2 ||\sigma(\bdf^{\ell-1}(\bx'))||_2$ is bounded and the angle between $\sigma(\bdf^{\ell-1}(\bx))$ and $\sigma(\bdf^{\ell-1}(\bx'))$ lies in $(0,\pi)$ meaning there is no collinearity, then we have an upper bound $\nu_{\mM}$ for $\nu$ depending on that two conditions. Therefore, we similar choose $T$ balls with radius $\delta$ to cover $\mM$ as in the proof for Theorem \ref{thm:1}.
Let $\{\Delta_i\}_{i=1}^{T}$ denote the centers of these balls. Then using \eqref{eq:thm1Bproof2} and union bounds, for any two centers, such as $\Delta_1$ and $\Delta_2$, with probability at least $1-2\exp\left(\log(T^2)-\frac{r\epsilon^2}{8\nu_{\mM}^2}\right)$,

\begin{equation}\label{eq:thm1Bproof3}
|\hat{\mK}^\ell(\Delta_1,\Delta_2)-\mK^\ell(\Delta_1,\Delta_2)| \leq \frac{\epsilon}{2}.
\end{equation}

Then similarly as in the proof for Theorem 0, 
for the function $\bu(\bx,\bx')=\hat{\mK}^\ell(\bdf^{\ell-1}(\bx),\bdf^{\ell-1}(\bx'))-\mK^\ell(\bdf^{\ell-1}(\bx),\bdf^{\ell-1}(\bx'))$, we have the inequality from partial derivative that

\begin{equation}\label{eq:thm1Bproof4}
|\bu(\bx,\bx')-\bu(\bx_0,\bx'_0)| \leq L_{\sigma_{\mK}}(||\sigma(\bdf^{\ell-1}(\bx))-\sigma(\bdf^{\ell-1}(\bx_0))||_2+||\sigma(\bdf^{\ell-1}(\bx'))-\sigma(\bdf^{\ell-1}(\bx'_0))||_2),
\end{equation}

where 

\begin{align*}
L_{\sigma_{\mK}} &= \arg\max_{\sigma(\bdf^{\ell-1}(\bx)),\sigma(\bdf^{\ell-1}(\bx'))\in \mM}
||\frac{1}{r}\sum_{i=1}^r \frac{\partial \sigma_{\mK}(\sigma(\bdf^{\ell-1}(\bx))^T\bw_i)}{\partial \sigma(\bdf^{\ell-1}(\bx))}\sigma_{\mK}(\sigma(\bdf^{\ell-1}(\bx'))^T\bw_i)\\
&-\mE_\bw \frac{\partial \sigma_{\mK}(\sigma(\bdf^{\ell-1}(\bx))^T\bw)}{\partial \sigma(\bdf^{\ell-1}(\bx))}\sigma_{\mK}(\sigma(\bdf^{\ell-1}(\bx'))^T\bw)||_2\\
&= 
||\frac{1}{r}\sum_{i=1}^r \frac{\partial \sigma_{\mK}(\sigma(\bdf^{\ell-1}(\bx^\ast))^T\bw_i)}{\partial \sigma(\bdf^{\ell-1}(\bx))}\sigma_{\mK}(\sigma(\bdf^{\ell-1}(\bx'^\ast))^T\bw_i)\\
&-\mE_\bw \frac{\partial \sigma_{\mK}(\sigma(\bdf^{\ell-1}(\bx^\ast))^T\bw)}{\partial \sigma(\bdf^{\ell-1}(\bx))}\sigma_{\mK}(\sigma(\bdf^{\ell-1}(\bx'^\ast))^T\bw)||_2.
\end{align*}

We also have that

\begin{align*}
\mE L_{\sigma_{\mK}}^2 
&=
\frac{1}{r^2}\sum_{i=1}^r\mE|| \frac{\partial \sigma_{\mK}(\sigma(\bdf^{\ell-1}(\bx^\ast))^T\bw_i)}{\partial \sigma(\bdf^{\ell-1}(\bx))}\sigma_{\mK}(\sigma(\bdf^{\ell-1}(\bx'^\ast))^T\bw_i)||_2^2\\
&-\frac{1}{r^2}||\mE_\bw \frac{\partial \sigma_{\mK}(\sigma(\bdf^{\ell-1}(\bx^\ast))^T\bw)}{\partial \sigma(\bdf^{\ell-1}(\bx))}\sigma_{\mK}(\sigma(\bdf^{\ell-1}(\bx'^\ast))^T\bw)||_2^2\\
&\leq
\frac{1}{r^2}\sum_{i=1}^r\mE|| \frac{\partial \sigma_{\mK}(\sigma(\bdf^{\ell-1}(\bx^\ast))^T\bw_i)}{\partial \sigma(\bdf^{\ell-1}(\bx))}\sigma_{\mK}(\sigma(\bdf^{\ell-1}(\bx'^\ast))^T\bw_i)||_2^2.
\end{align*}

Since $\sigma_{\mK}$ is ReLU, $||\sigma_{\mK}'||_\infty\leq 1$, therefore we get that

$$\mE L_{\sigma_{\mK}}^2 \leq \frac{1}{r}\mE|\sigma_{\mK}(\sigma(\bdf^{\ell-1}(\bx'^\ast))^T\bw)|||\bw||_2^2.$$

Again, since $\bw$ follows a normal distribution that is symmetric, we choose axis to satisfy $\sigma(\bdf^{\ell-1}(\bx'^\ast)) = \be_1 ||\sigma(\bdf^{\ell-1}(\bx'^\ast))||_2$. Then we do a calculation,

\begin{align*}
\mE L_{\sigma_{\mK}}^2 
&\leq \frac{1}{r}||\sigma(\bdf^{\ell-1}(\bx'^\ast))||_2 \int_0^\infty d w_1 \left\{\int_{-\infty}^\infty d w_2 ...\int_{-\infty}^\infty d w_d (w_1\sum_{j=1}^d w_j^2)p(w_2,...,w_d)\right\}p(w_1)\\
&= \frac{1}{r}||\sigma(\bdf^{\ell-1}(\bx'^\ast))||_2 \int_0^\infty d w_1 \{w_1^3 + (d-1)w_1)p(w_1)\}\\
&= \frac{1}{r}||\sigma(\bdf^{\ell-1}(\bx'^\ast))||_2 \int_0^\infty d w_1 \{(d+3)w_1p(w_1)\}\\
&= \frac{1}{r}||\sigma(\bdf^{\ell-1}(\bx'^\ast))||_2 \frac{\sqrt{2}(d+3)}{\sqrt{\pi}}\\
&\leq \frac{\sqrt{2}(d+3)c_{\mM}}{\sqrt{\pi}r}
\end{align*}

Therefore, by Markov's inequality,

$$\mP(L_{\sigma_{\mK}} \geq \frac{\epsilon}{4 \delta}) \leq \mE L_{\sigma_{\mK}}^2 \frac{16\delta^2}{\epsilon^2} \leq \frac{16\sqrt{2}(d+3)c_{\mM}\delta^2}{\sqrt{\pi}\epsilon^2 r}.$$

Then using Eq. \eqref{eq:thm1Bproof4}, with probability at least $1-\frac{16\sqrt{2}(d+3)c_{\mM}\delta^2}{\sqrt{\pi}\epsilon^2 r}$,

$$|u(\bx,\bx')-u(\bx_0,\bx'_0)|\leq \frac{\epsilon}{2}.$$

This inequality combined with Eq. \eqref{eq:thm1Bproof3} enables us to conclude that

$$\sup_{\sigma(\bdf^{\ell-1}(\bx)), \sigma(\bdf^{\ell-1}(\bx'))\in \mM}
|\hat{\mK}^\ell(\bdf^{\ell-1}(\bx),\bdf^{\ell-1}(\bx'))-\mK^\ell(\bdf^{\ell-1}(\bx),\bdf^{\ell-1}(\bx'))| \leq \epsilon .$$

with probability at least $1-\frac{16\sqrt{2}(d+3)c_{\mM}\delta^2}{\sqrt{\pi}\epsilon^2 r}-2\exp\left(\log(T^2)-\frac{r\epsilon^2}{8\nu_{\mM}^2}\right)$. 
Recall that $T = (4 diam(\mM)/\delta)^d$, so the probability has a format of $1-\kappa_1 \delta^2 - \kappa_2 \delta^{-2d}$ for $\delta$. By setting $\delta = \frac{\kappa_2}{\kappa_1}^{\frac{1}{2+2d}}$, we have the probability as $1-2\kappa_1^{\frac{2d}{2+2d}}\kappa_2^{\frac{2}{2+2d}}$. So the probability is at least

$$1-2^{10}\left(\frac{c_{\mM} (d+3) diam(\mM)^2}{\epsilon^2 r}\right)^{\frac{d}{1+d}}\exp\left\{-\frac{r\epsilon^2}{8(1+d)\nu_{\mM}^2}\right\}.$$

\end{proof}

\subsubsection{The relation between random feature \cite{Cut17} and inducing points approximation \cite{Sal17}}
First, we review the algorithm in \cite{Sal17} based on inducing points and doubly stochastic variational inference.

In the background section, we introduced that a $L$ layer DGP can be represented by

$$p(\by,\{\bF^\ell\}_{\ell=1}^L)=\prod_{i=1}^n p(y_i|f^L_i)\prod_{\ell=1}^L p(\bF^\ell|\bF^{\ell-1}),$$

where $\bF^\ell\in \mR^{n\times d}$ for $0\leq \ell<L$ and $\bF^L\in \mR^{n\times 1}$, with $\bF^\ell|\bF^{\ell-1}\sim N(0,\mK^\ell(\bF^{\ell-1},\bF^{\ell-1}))$.

In \cite{Sal17}, they further define an additional set of $m$ inducing points $\bZ^\ell = (\bz^\ell_1, ..., \bz^\ell_m)$ for each layer $0\leq \ell<L$. We use the notation $\bu^\ell = f^\ell(\bZ^{\ell-1})$ for the function values at the inducing points. Since we have $d$ output on layer $\ell$, we use $\bU^\ell\in \mR^{m\times d}$ for the function value matrix at the inducing points. By the definition of GP, the joint density $p(\bF^\ell, \bU^\ell)$ is a Gaussian distribution given inputs from previous layer. Therefore, we have the joint posterior of $\by, \{\bF^\ell, \bU^\ell\}_{\ell=1}^L$ is 

$$p(\by,\{\bF^\ell,\bU^\ell\}_{\ell=1}^L)=\prod_{i=1}^n p(y_i|f^L_i)\prod_{\ell=1}^L p(\bF^\ell|\bU^\ell;\bF^{\ell-1},\bZ^{\ell-1})p(\bU^\ell;\bZ^{\ell-1}).$$

The posterior of $\{\bF^\ell, \bU^\ell\}_{\ell=1}^L$ is intractable, so the authors in \cite{Sal17} define the variational posterior

$$q(\{\bF^\ell,\bU^\ell\}_{\ell=1}^L)=\prod_{\ell=1}^L p(\bF^\ell|\bU^\ell;\bF^{\ell-1},\bZ^{\ell-1})q(\bU^\ell),$$

with $q(\bU^\ell)=\prod_{j=1}^d q(\bU^\ell_{.j})$ and $q(\bU^\ell_{.j})\sim N(\bm^\ell_j, \bS^\ell_j)$.
Then they calculate the evidence lower bound of the DGP, which is

$$\ELBO_{DGP} = \mE_{q(\{\bF^\ell,\bU^\ell\}_{\ell=1}^L)}\left[\frac{p(\by, \{\bF^\ell,\bU^\ell\}_{\ell=1}^L}{q(\{\bF^\ell,\bU^\ell\}_{\ell=1}^L)}\right].$$

Based on the definition of $q(\{\bF^\ell,\bU^\ell\}_{\ell=1}^L)$, we can simplify $\ELBO_{DGP}$ and show that it is equal as

\begin{equation}\label{eq:thm2proof1}
\ELBO_{DGP}=\sum_{i=1}^n \mE_{q(f^L_i)}[\log p(y_i|f^L_i)]-\sum_{\ell=1}^L KL[q(\bU^\ell)|p(\bU^\ell;\bZ^{\ell-1})].
\end{equation}

From \cite{Sal17}, after marginalizing the inducing variables from each layer analytically, we can show that

\begin{equation}\label{eq:thm2proof2}
q(\{\bF^\ell\}_{\ell=1}^L)=\prod_{\ell=1}^L q(\bF^\ell|\bm^\ell,\bS^\ell;\bF^{\ell-1},\bZ^{\ell-1})=\prod_{\ell=1}^L N(\bF^\ell|\bmu^\ell,\bSigma^\ell) = \prod_{\ell=1}^L\prod_{j=1}^d N(\bF^\ell_{.j}|\tilde{\bmu}^\ell_j, \tilde{\bSigma}^\ell_j).
\end{equation}

Here,

$$\tilde{\bmu}^\ell_j = \balpha(\bF^{\ell-1})^T\bm^\ell_{.j}$$

$$\tilde{\bSigma}^\ell_j = \mK^\ell(\bF^{\ell-1},\bF^{\ell-1})-\balpha(\bF^{\ell-1})^T(\mK^\ell(\bZ^{\ell-1},\bZ^{\ell-1})-\bS_{.j})\balpha(\bF^{\ell-1})$$

with $\balpha(\bF^{\ell-1}) = \mK^\ell(\bZ^{\ell-1},\bZ^{\ell-1})^{-1}\mK^\ell(\bZ^{\ell-1},\bF^{\ell-1}).$

From Eq. \eqref{eq:thm2proof1}, we only need to get $q(f^L_i)$ from $q(\{\bF^\ell\}_{\ell=1}^L)$ for sample $i$ with $1\leq i\leq n$. In \cite{Sal17}, they point out that based on the format of Eq. \eqref{eq:thm2proof2},

$$q(f^L_i) = \int\cdot\cdot\cdot\int \prod_{\ell=1}^L q(\bF^\ell_{i.}|\bm^\ell_{i.}, \bS^\ell_{i,}; \bF^{\ell-1}_{i.}, \bZ^{\ell-1})d\bF^{\ell-1}_{i.},$$

which means that the $i$th marginal of the final layer of the variational DGP for sample $i$ depends only on the $i$th marginals of all the other layers.\\

{\bf Proof of Theorem 2}

\begin{theorem}
Using the variational approximation \cite{Sal17} for the posterior of a DGP defined on $\{\hat{\mK}^\ell\}_{\ell=1}^L$ with inducing points, we obtain exactly the same variational posterior $q(\{\bF^\ell\}_{\ell=1}^L)$ and evidence lower bound $\ELBO$ 
as the variational posterior for $\mN(\mS)$. 
\end{theorem}
\begin{proof}
To show the equivalence of evidence lower bound, we only need to guarantee that
$q(\bF^\ell_{i.}|\bm^\ell_{i.}, \bS^\ell_{i,}; \bF^{\ell-1}_{i.}, \bZ^{\ell-1})$ and $KL[q(\bU^\ell)|p(\bU^\ell;\bZ^{\ell-1})]$ are the same as the relevant values for $\mN(\mS)$ for all $1\leq i\leq n$ and $1\leq \ell \leq L$. We also need to show the equivalence between variational posterior $q(\{\bF^\ell\}_{\ell=1}^L)$ for the two methods. All those can be satisfied by showing the equivalence that $q(\bF^\ell|\bm^\ell, \bS^\ell; \bF^{\ell-1}, \bZ^{\ell-1})$ and $KL[q(\bU^\ell)|p(\bU^\ell;\bZ^{\ell-1})]$ are the same as the relevant values for $\mN(\mS)$ for all $1\leq \ell \leq L$. For both two methods, since for each layer $\ell$, the $d$ outputs are independent, so the posterior distribution can be decomposed into a product of $d$ terms and the KL divergence can be decomposed into a summation of $d$ terms. We only need to prove the result for a single $j$ with $1\leq j\leq d$ and a single $\ell$ with $1\leq \ell\leq L$. 

Based on Eq. \eqref{eq:thm2proof2}, for $q(\bF^\ell_{.j}|\bm^\ell_{.j}, \bS^\ell_{.j}; \bF^{\ell-1}, \bZ^{\ell-1})$, the mean and variances are

$$\tilde{\bmu}^\ell_j = \balpha(\bF^{\ell-1})^T\bm^\ell_{.j}$$

$$\tilde{\bSigma}^\ell_j = \hat{\mK}^\ell(\bF^{\ell-1},\bF^{\ell-1})-\balpha(\bF^{\ell-1})^T(\hat{\mK}^\ell(\bZ^{\ell-1},\bZ^{\ell-1})-\bS_{.j})\balpha(\bF^{\ell-1})$$

with $\balpha(\bF^{\ell-1}) = \hat{\mK}^\ell(\bZ^{\ell-1},\bZ^{\ell-1})^{-1}\hat{\mK}^\ell(\bZ^{\ell-1},\bF^{\ell-1}).$
We can decompose the kernel into $\hat{\mK}^\ell(\bZ^{\ell-1},\bZ^{\ell-1}) = \Phi^\ell(\bZ^{\ell-1})\Phi^\ell(\bZ^{\ell-1})^T$. Now we choose the number of inducing points $m$ as $m=r$, then we have a square matrix $\Phi^\ell(\bZ^{\ell-1})$ with each entry is independently identically from a distribution based on the random feature weight vector $\bw_j$ and the random inducing points $\bZ^{\ell-1}_{i.}$.

For continuous $\sigma_{\mK}$ and $\sigma$, every entry in $\Phi^\ell(\bZ^{\ell-1})$ is absolutely continuous with respect to Lebesgue measure since we can define density function. Then based on random matrix theory \cite{Rud08, Tao12},  the square matrix $\Phi^\ell(\bZ^{\ell-1})$ is almost surely invertible. Therefore, we treat $\Phi^\ell(\bZ^{\ell-1})$ as an invertible matrix in following analysis.

Replace $\hat{\mK}^\ell(\bZ^{\ell-1},\bZ^{\ell-1})$ by $\Phi^\ell(\bZ^{\ell-1})\Phi^\ell(\bZ^{\ell-1})^T$, 
we have that

$$\tilde{\bmu}^\ell_j = \Phi^\ell(\bF^{\ell-1})(\Phi^\ell(\bZ^{\ell-1})^T\Phi^\ell(\bZ^{\ell-1}))^{-1}\Phi^\ell(\bZ^{\ell-1})^T\bm^\ell_{.j}$$

\begin{equation*}
\tilde{\bSigma}^\ell_j =\Phi^\ell(\bF^{\ell-1})\Phi^\ell(\bF^{\ell-1})^T- \balpha(\bF^{\ell-1})^T(\Phi^\ell(\bZ^{\ell-1})\Phi^\ell(\bZ^{\ell-1})^T-\bS_{.j})\balpha(\bF^{\ell-1})
\end{equation*}

Through simple algebra, we get that

\begin{equation}\begin{aligned}\label{eq:thm2proof3}
\tilde{\bmu}^\ell_j &= \balpha(\bF^{\ell-1})^T\bm^\ell_{.j}\\
\tilde{\bSigma}^\ell_j &= \balpha(\bF^{\ell-1})^T\bS_{.j}\balpha(\bF^{\ell-1})
\end{aligned}\end{equation}

Since $\Phi^\ell(\bZ^{\ell-1})$ is invertible, we define 

$$\bm^\ell_{.j} = \Phi^\ell(\bZ^{\ell-1}) \bmu^\ell_{j, new}$$

$$\bS_{.j} = \Phi^\ell(\bZ^{\ell-1}) \bSigma^\ell_{j,new} \Phi^\ell(\bZ^{\ell-1})^T,$$

and plug them into Eq. \eqref{eq:thm2proof3} then we get

\begin{equation}\begin{aligned}\label{eq:IPkey1}
&\tilde{\bmu}^\ell_j = \Phi^\ell(\bF^{\ell-1}) \bmu^\ell_{j, new}\\
&\tilde{\bSigma}^\ell_j = \Phi^\ell(\bF^{\ell-1}) \bSigma^\ell_{j,new} \Phi^\ell(\bF^{\ell-1})^T.
\end{aligned}\end{equation}

In our BNN construction for $\mN(\mS)$, the variational posterior over $\bV$ leads to $\bF^\ell_{.j} = \Phi^\ell(\bF^{\ell-1})\bv^\ell_j$ with $\bv^\ell_j\sim N(\bmu^\ell_{j, new}, \bSigma^\ell_{j,new})$. Then we have that

$$\bF^\ell_{.j} \sim N(\Phi^\ell(\bF^{\ell-1}) \bmu^\ell_{j, new}, \Phi^\ell(\bF^{\ell-1}) \bSigma^\ell_{j,new} \Phi^\ell(\bF^{\ell-1})^T).$$

That is identical with the results from inducing points method in Eq. \eqref{eq:IPkey1}. Based on this construction, we also have that

$$\bU^\ell_{.j} = \Phi^\ell(\bZ^{\ell-1})\bv^\ell_j$$

Since the KL divergence is invariant under parameter transformations, we have that

$$KL(q(\bU^\ell_{.j})||p(\bU^\ell_{.j})) = KL(q(\bv^\ell_j)||p(\bv^\ell_j))$$

\end{proof}

{\bf Proof of Theorem 3}

For a kernel $\mK^\ell$ belongs to a general class, we can still use a similar technique as in the proof of Theorem 2 to show the equivalence. However, this time we cannot use the random feature matrix $\Phi^\ell(\bF^{\ell-1})\Phi^\ell(\bF^{\ell-1})^T$ to approximate $\mK^\ell(\bF^{\ell-1},\bF^{\ell-1})$. It turns out that a good replacement for $\Phi^\ell(\bF^{\ell-1})$ to approximate the basis of $\mK^\ell(\bF^{\ell-1},\bF^{\ell-1})$ is $\mK^\ell(\bF^{\ell-1},\bZ^{\ell-1})\mK^\ell(\bZ^{\ell-1},\bZ^{\ell-1})^{-1/2}$ which we will show shortly. The proof technique for Theorem 3 is similar as the technique for Theorem 2. However, for a general class of $\mK^\ell$, $\mK^\ell(\bF^{\ell-1},\bF^{\ell-1})$ can be full rank which is equal to sample size $n$. Therefore, the difference from the approximation using the rank $r$ basis $\mK^\ell(\bF^{\ell-1},\bZ^{\ell-1})\mK^\ell(\bZ^{\ell-1},\bZ^{\ell-1})^{-1/2}$ is $\mK^\ell(\bF^{\ell-1},\bF^{\ell-1})-\mK^\ell(\bF^{\ell-1},\bZ^{\ell-1})\mK^\ell(\bZ^{\ell-1},\bZ^{\ell-1})^{-1}\mK^\ell(\bZ^{\ell-1},\bF^{\ell-1})$. This is the constant offset that does not depend on training which we mention in Theorem 3. For the optimization of $\ELBO$, only the diagonal terms in this offset matrix is used so we can also add this into BNN as a bias term with random weight that we do not train. 

{\bf Remark.} After the optimization of $\ELBO$, one can get the uncertainty estimates from the variational posterior. If  $\mK^\ell(\bF^{\ell-1},\bF^{\ell-1})-\mK^\ell(\bF^{\ell-1},\bZ^{\ell-1})\mK^\ell(\bZ^{\ell-1},\bZ^{\ell-1})^{-1}\mK^\ell(\bZ^{\ell-1},\bF^{\ell-1})$ is not present, then one can choose $\bV$ from its variational posterior and the output estimates for every samples directly come from one pass of feed-forward neural network. However, when $\mK^\ell(\bF^{\ell-1},\bF^{\ell-1})-\mK^\ell(\bF^{\ell-1},\bZ^{\ell-1})\mK^\ell(\bZ^{\ell-1},\bZ^{\ell-1})^{-1}\mK^\ell(\bZ^{\ell-1},\bF^{\ell-1})$ exists, $n$ passes of feed-forward neural network computation for $n$ samples need to depend on each other to derive the outputs. In \cite{Sal17}, they also permit the prior of DGP to have a non-zero mean function, which can lead to another offset if the prior mean of DGP at each layer is non-zero. Similarly, this offset does not depend on training and can be included into BNN as a bias term with random weights that we do not train. Another term that is usually discussed in DGP is the noisy corruption. In this result for general kernel $\mK$ in Thoerem 3, in \cite{Sal17}, the authors show that the noisy corruption can be included into the kernel $\mK$. For our earlier result in Theorem 2, we do not include the noisy corruption in intermediate layers, since the complexity of intermediate function is already restricted by the rank $r$. It does not overfit the data so we do not need the noisy corruption which is usually used to avoid overfitting when the kernel basis has infinite dimension which can be super expressive.

Now we review the definition of $IPB$ and Theorem 3, then we present the proof for Theorem 3.
\begin{definition}
For a kernel $\mK$, $IPB$ can be constructed by choosing $r$ additional points $\bZ$ (inducing points), taking the inputs $\bx$ and outputting an $r$-dimension vector $\mK(\bx,\bZ)\mK(\bZ,\bZ)^{-1/2}$.  
\end{definition}
\begin{theorem}
Using the variational approximation \cite{Sal17} for the posterior of a DGP defined on $\{\mK^\ell\}_{\ell=1}^L$ with inducing points, we can obtain the same variational posterior $q(\{\bF^\ell\}_{\ell=1}^L)$ and evidence lower bound ELBO as the variational posterior for $\mN(\mS)$ (with $IPB$) except a constant offset that does not depend on training. 
\end{theorem}
\begin{proof}
To show the equivalence of evidence lower bound, we only need to guarantee that
$q(\bF^\ell_{i.}|\bm^\ell_{i.}, \bS^\ell_{i,}; \bF^{\ell-1}_{i.}, \bZ^{\ell-1})$ and $KL[q(\bU^\ell)|p(\bU^\ell;\bZ^{\ell-1})]$ are the same as the relevant values for $\mN(\mS)$ for all $1\leq i\leq n$ and $1\leq \ell \leq L$. We also need to show the equivalence between variational posterior $q(\{\bF^\ell\}_{\ell=1}^L)$ for the two methods. All those can be satisfied by showing the equivalence that $q(\bF^\ell|\bm^\ell, \bS^\ell; \bF^{\ell-1}, \bZ^{\ell-1})$ and $KL[q(\bU^\ell)|p(\bU^\ell;\bZ^{\ell-1})]$ are the same as the relevant values for $\mN(\mS)$ for all $1\leq \ell \leq L$. For both two methods, since for each layer $\ell$, the $d$ outputs are independent, so the posterior distribution can be decomposed into a product of $d$ terms and the KL divergence can be decomposed into a summation of $d$ terms. We only need to prove the result for a single $j$ with $1\leq j\leq d$ and a single $\ell$ with $1\leq \ell\leq L$. 

Based on Eq. \eqref{eq:thm2proof2}, for $q(\bF^\ell_{.j}|\bm^\ell_{.j}, \bS^\ell_{.j}; \bF^{\ell-1}, \bZ^{\ell-1})$, the mean and variances are

$$\tilde{\bmu}^\ell_j = \balpha(\bF^{\ell-1})^T\bm^\ell_{.j}$$

$$\tilde{\bSigma}^\ell_j = \mK^\ell(\bF^{\ell-1},\bF^{\ell-1})-\balpha(\bF^{\ell-1})^T(\mK^\ell(\bZ^{\ell-1},\bZ^{\ell-1})-\bS_{.j})\balpha(\bF^{\ell-1})$$

with $\balpha(\bF^{\ell-1}) = \mK^\ell(\bZ^{\ell-1},\bZ^{\ell-1})^{-1}\mK^\ell(\bZ^{\ell-1},\bF^{\ell-1}).$

We use $r$ to refer the number of inducing points and {\bf we denote the $IPB$ block matrix} $\mK^\ell(\bF^{\ell-1},\bZ^{\ell-1})\mK^\ell(\bZ^{\ell-1},\bZ^{\ell-1})^{-1/2}$ as $\IP^\ell(\bF^{\ell-1})$, then we notice that

\begin{equation}\begin{aligned}\label{eq:thm3proof3}
\tilde{\bmu}^\ell_j &= \IP^\ell(\bF^{\ell-1})\mK^\ell(\bZ^{\ell-1},\bZ^{\ell-1})^{-1/2}\bm^\ell_{.j}\\
\tilde{\bSigma}^\ell_j &=\IP^\ell(\bF^{\ell-1})\mK^\ell(\bZ^{\ell-1},\bZ^{\ell-1})^{-1/2}\bS_{.j}\mK^\ell(\bZ^{\ell-1},\bZ^{\ell-1})^{-1/2}\IP^\ell(\bF^{\ell-1})^T\\
&+\mK^\ell(\bF^{\ell-1},\bF^{\ell-1})-\mK^\ell(\bF^{\ell-1},\bZ^{\ell-1})\mK^\ell(\bZ^{\ell-1},\bZ^{\ell-1})^{-1}\mK^\ell(\bZ^{\ell-1},\bF^{\ell-1}).
\end{aligned}\end{equation}

We have discussed the second term in $\tilde{\bSigma}^\ell_j$ which is a constant offset that does not depend on training. Therefore we assume it to be zero in following analysis then we get the exactly same result as our variational posterior approximation for $\mN(\mS)$ when $IPB$ is used.

We define 

$$\bm^\ell_{.j} = \mK^\ell(\bZ^{\ell-1},\bZ^{\ell-1})^{1/2}\bmu^\ell_{j, new}$$

$$\bS_{.j} = \mK^\ell(\bZ^{\ell-1},\bZ^{\ell-1})^{1/2}\bSigma^\ell_{j,new}\mK^\ell(\bZ^{\ell-1},\bZ^{\ell-1})^{1/2},$$

and plug them into Eq. \eqref{eq:thm3proof3} then we get

\begin{equation}\begin{aligned}\label{eq:IPkey2}
&\tilde{\bmu}^\ell_j = \IP^\ell(\bF^{\ell-1}) \bmu^\ell_{j, new}\\
&\tilde{\bSigma}^\ell_j = \IP^\ell(\bF^{\ell-1}) \bSigma^\ell_{j,new} \IP^\ell(\bF^{\ell-1})^T.
\end{aligned}\end{equation}

In our BNN construction for $\mN(\mS)$, the variational posterior over $\bV$ leads to $\bF^\ell_{.j} = \IP^\ell(\bF^{\ell-1}) \bv^\ell_j$ with $\bv^\ell_j\sim N(\bmu^\ell_{j, new}, \bSigma^\ell_{j,new})$. Then we have that

$$\bF^\ell_{.j} \sim N(\IP^\ell(\bF^{\ell-1})  \bmu^\ell_{j, new}, \IP^\ell(\bF^{\ell-1})  \bSigma^\ell_{j,new} \IP^\ell(\bF^{\ell-1})^T).$$

That is identical with the results from inducing points method in Eq. \eqref{eq:IPkey2}. Based on this construction, we also have that

$$\bU^\ell_{.j} = \IP^\ell(\bZ^{\ell-1})\bv^\ell_j=\mK^\ell(\bZ^{\ell-1},\bZ^{\ell-1})^{1/2}\bv^\ell_j$$

Since the KL divergence is invariant under parameter transformations, we have that

$$KL(q(\bU^\ell_{.j})||p(\bU^\ell_{.j})) = KL(q(\bv^\ell_j)||p(\bv^\ell_j))$$

\end{proof}

\subsubsection{Proof of Theorem 4}
The post-training ANOVA decomposition is

\begin{equation}\label{eq:thm4}
\I^n_T(\bx_T) = \prod_{i\in T} (I_{x_i}-\mE^n_{x_i})\prod_{j\notin T}\mE^n_{x_j}f(x_1,...,x_p),
\end{equation}

\begin{theorem}
 If there exist inputs clusters $\{T^\ast_j\}_{j=1}^{k^\ast}$ such that $f^\ast(\bx) = \sum_{j=1}^{k^\ast}g^\ast_j(\bx_{T^\ast_j})$ with $k^\ast$ at the order of polynomial in $p$ and $c = \max_{j=1}^{k^\ast}|T^\ast_j| = O(\log p)$, then there exists a trained AddNN that predicts $\by$ well and restricts the number of possible interactions at polynomial in $p$.
Further, if every sub neural network has $L$ layers with $d$ hidden units, then the computation complexity of measure \eqref{eq:thm4} is at most $n^ck^\ast d^{2L-1}$, which is also polynomial in $p$. 
\end{theorem}

\begin{proof}
In that case, there exists a trained AddNN $f$ which is 
$\hat{f}(\bx)=\sum_{j=1}^{k^\ast}\hat{g}_j(\bx_{T^\ast_j})$ with the same $k^\ast$ and inputs clusters $\{T^\ast_j\}_{j=1}^{k^\ast}$. For such $\hat{f}$, without knowing the truth, every possible interaction is a subset of one inputs cluster $T^\ast_j$ for some $j$. Therefore, the number of possible interactions is bounded by $\sum_{j=1}^{k^\ast}2^{|T^\ast_j|}\leq k^\ast 2^c$, so it is polynomial in $p$. For an interaction among a subset $S$, we have that

\begin{align*}
\I^n_S(\bx_S) &= \prod_{i\in S} (I_{x_i}-\mE^n_{x_i})\prod_{j\notin S}\mE^n_{x_j}\hat{f}(x_1,...,x_p)\\
&= \prod_{i\in S} (I_{x_i}-\mE^n_{x_i})\prod_{j\notin S}\mE^n_{x_j}\sum_{m=1}^{k^\ast}\hat{g}_m(\bx_{T^\ast_m})\\
&= \sum_{m=1}^{k^\ast} \prod_{i\in S} (I_{x_i}-\mE^n_{x_i})\prod_{j\notin S}\mE^n_{x_j}\hat{g}_m(\bx_{T^\ast_m})\\
& = \sum_{m:S\subseteq T^\ast_m} \prod_{i\in S} (I_{x_i}-\mE^n_{x_i})\prod_{j\notin S}\mE^n_{x_j}\hat{g}_m(\bx_{T^\ast_m}),
\end{align*}

because for $m$ that $S\not\subseteq T^\ast_m$, then there exists an $i_0\in S$ such that $i_0\notin T^\ast_m$, then 

$$\prod_{i\in S}(I_{x_i}-\mE^n_{x_i})\prod_{j\notin S}\mE^n_{x_j}\hat{g}_m(\bx_{T^\ast_m})=0$$

because $(I_{x_{i_0}}-\mE^n_{x_{i_0}})\hat{g}_m(\bx_{T^\ast_m})=\hat{g}_m(\bx_{T^\ast_m})-\hat{g}_m(\bx_{T^\ast_m})=0$.
We can further simply $\I^n_S(\bx_S)$,

\begin{align*}
\I^n_S(\bx_S) &= \sum_{m:S\subseteq T^\ast_m} \prod_{i\in S} (I_{x_i}-\mE^n_{x_i})\prod_{j\notin S}\mE^n_{x_j}\hat{g}_m(\bx_{T^\ast_m})\\
&= \sum_{m:S\subseteq T^\ast_m} \prod_{i\in S} (I_{x_i}-\mE^n_{x_i})\prod_{j\notin S,j\in T^\ast_m}\mE^n_{x_j}\hat{g}_m(\bx_{T^\ast_m}).
\end{align*}

Therefore, the computation complexity of $\I^n_S(\bx_S)$ is equal to the computation complexity for 

\begin{equation}\label{eq:thm4union}
\cup_{m:S\subseteq T^\ast_m}\left\{\prod_{i\in S} (I_{x_i}-\mE^n_{x_i})\prod_{j\notin S,j\in T^\ast_m}\mE^n_{x_j}\hat{g}_m(\bx_{T^\ast_m})\right\},
\end{equation}

which is the union of $|m:S\subseteq T^\ast_m|$ functions where each function involves the calculation of feed-forward neural network and empirical evaluation. Therefore, to compute all interactions, we need to compute the union of \eqref{eq:thm4union} for all possible interactions $S$. Because every possible interaction is a subset of one inputs cluster $T^\ast_j$ for some $j$, we can exchange the order of unions and we get that the computation complexity for all interactions is equal as evaluating

\begin{equation}
\cup_{m}\cup_{S\subseteq T^\ast_m}\left\{\prod_{i\in S} (I_{x_i}-\mE^n_{x_i})\prod_{j\notin S,j\in T^\ast_m}\mE^n_{x_j}\hat{g}_m(\bx_{T^\ast_m})\right\}.
\end{equation}

To compute $\cup_{S\subseteq T^\ast_m}\left\{\prod_{i\in S} (I_{x_i}-\mE^n_{x_i})\prod_{j\notin S,j\in T^\ast_m}\mE^n_{x_j}\hat{g}_m\left(\bx_{T^\ast_m}\right)\right\}$ for some $m$, we only need to compute

\begin{equation*}
M^n_{T^\ast_m}(\bx_{T^\ast_m}) = \prod_{i\in T^\ast_m} I_{x_i}\hat{g}\left(\bx_{T^\ast_m}\right),
\end{equation*}

which involves all the evaluations of the feed-forward neural network. To evaluate $\prod_{i\in S} (I_{x_i}-\mE^n_{x_i})\prod_{j\notin S,j\in T^\ast_m}\mE^n_{x_j}\hat{g}_m\left(\bx_{T^\ast_m}\right)$ for $S\subseteq T^\ast_m$, the computation only involves basic addition operation given $M^n_{T^\ast_m}$. Therefore, we show that the evaluation of all possible interactions has the same computation complexity as evaluating

\begin{equation}
\cup_m \{M^n_{T^\ast_m}(\bx_{T^\ast_m})\} = \cup_m \{\prod_{i\in T^\ast_m} I_{x_i}\hat{g}\left(\bx_{T^\ast_m}\right)\}.
\end{equation}

For every member in the union regarding $m$, the evaluation complexity is $n^{|T^\ast_m|}d^{2L-1}\leq n^c d^{2L-1}$. Therefore, the computation complexity regarding \eqref{eq:thm4} for all possible interactions based on model $\hat{f}$ is bounded by $n^ck^\ast d^{2L-1}$, which is polynomial in $p$ when $k^\ast$ is polynomial in $p$ and $c$ is the order of $O(\log p)$. 
\end{proof}

{\bf Remark 1.} When the truth function $f^\ast$ is additive, the function class of AddNN includes a member that can compute interactions and outputs from the model efficiently while the function class of an arbitrary NN make it impossible to learn such a model and always involve computation that is exponential in $p$. In practice, we always use group Lasso type penalty on the first layer to encourage each sub neural network to depend on few inputs to approach the truth function $f^\ast$.

{\bf Remark 2.} We can use the measure in \cite{Zei14} as well, which can be seen as choosing one sample baseline instead of the average baseline in \eqref{eq:thm4}. In other words, now we use operation $\delta_{x_i}(\bx^0_i)$ to replace $\mE^n_{x_i}$ in \eqref{eq:thm4} based on one sample $\bx^0$ for $1\leq i\leq n$. Then the computation complexity of measure \eqref{eq:thm4} does not depend on $n$ both for AddNN and NN. In that case, the number of possible interactions for AddNN is still at polynomial in $p$ and the number is exponential in $p$ for an arbitrary NN. Also, the computation complexity of measure \eqref{eq:thm4} is $k^\ast d^{2L-1}$ for AddNN and $(k^\ast d)^{2L-1}$ for an arbitrary NN with $L$ layers and $k^\ast d$ hidden units where the NN requires ${k^\ast}^{2(L-1)}$ times more computation. The choice of $\mE^n$  as baseline, compared to $\delta(\bx^0)$, is better for comparison with the population and can lead to a useful measure $||I^n_T(\bx_T)||_{2,n}$ (the empirical $\ell_2$ norm of the interaction), which can be used to detect the interactions. We give an example of their difference for explanation in the decision making process. For one who is interested in making an investment with $x_1$ dollars, the $\mE^n$ baseline informs that, based on this investment, how much more one can earn than the average of the money that people earn. On the other side, the $\delta(\bx_0)$ baseline informs that, at the current investment with $x_0$ dollars, if all other factors do not change, how much more one can earn if he/she decides to invest $x_1-x_0$ more dollars.

\subsection{Model details for experiments}
We discuss more details about Bayesian additive Neural Network (AddNN) implementation and provide more experimental results. In our experiments, we use 10 small(sub) neural networks, where each has 2 hidden layers. 

\textbf{Implementation Details:}
Our Bayesian Neural Network is a sum of 10 small neural networks and each small network consists of 3 layers. An input feature vector is passed through 10 sub-neural networks followed by addition operation to give a final scalar output. For each sub-neural network, we use 2 to 5 neurons for the first hidden layer and 5 to 20 neurons for the second hidden layer. We train the Bayesian Neural Network with batch size $=100$ and $0.01$ initial learning rate with exponential decay until the validation error converges. In order to pick sparse interpretable variables, we impose group Lasso for the first layer with respect to each input neuron, which associates with sub neural network. The group Lasso penalty hyper-parameter depends on the sparsity and addition structure. In our experiments, it ranges from $0.001$ to $1.0$. 

\textbf{AddNN learns the sparse additive structure:} To show our Bayesian additive neural network can learn the sparse addition structure of the function and the interaction, we provide one example of learning Friedman function $f_1$. We plot the learned matrix of the input layer and the first hidden layer, which can be seen in Fig \ref{fig:finteract}.
\begin{figure}
	\includegraphics[width=0.9\textwidth]{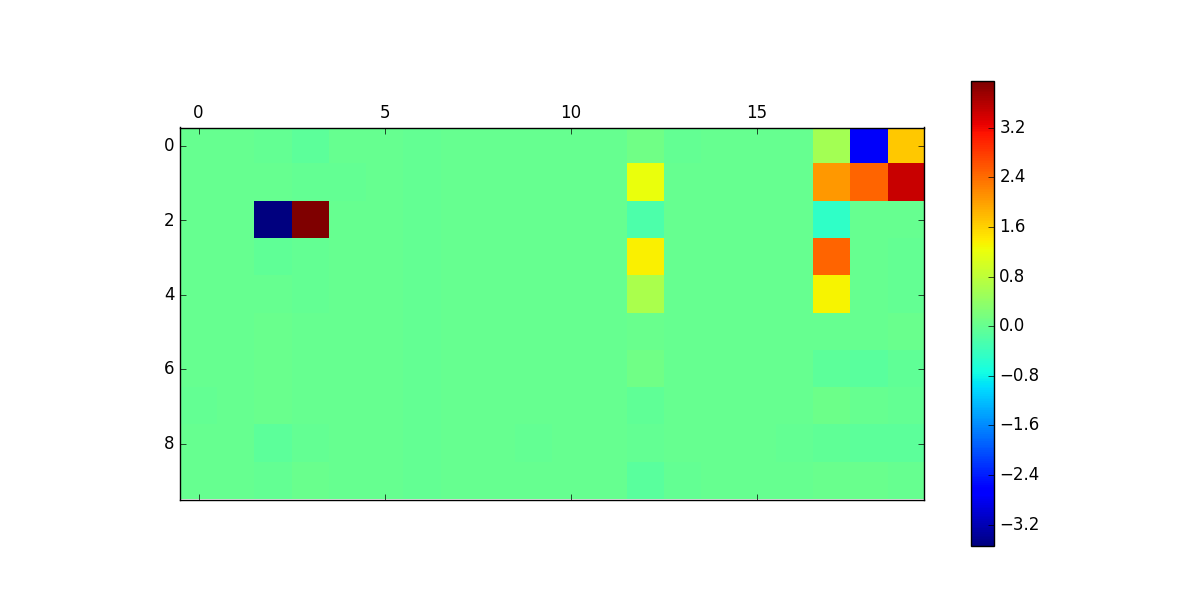}
	\caption{\label{fig:finteract} AddNN learns the additive structure of the function $f_1$. The rows of visualized matrix correspond to the neurons of the input layer and the columns correspond to the $10$ sub neural networks, each has 2 hidden units. The sparsity of the matrix demonstrates the sparse interaction between the neurons of the input layer.}		
\end{figure}

\textbf{The complete comparison for benchmark datasets:} We train a Bayesian additive neural network on benchmark datasets and evaluate each model on all benchmark datasets. Table \ref{tab:realdata2} shows that our proposed model (compact) offers favorable performance comparing with other methods for prediction accuracy. Due to space limitation, we only include the MC dropout baseline in the main body.

\begin{table}
	\caption{Average test performance in RSME and Standard Errors for AddNN(ours), dropout uncertainty (MC dropout), deep Gaussian process (DGP 5) and probabilistic back-propagation(PBP) on benchmarks. Dataset size($N$) and input dimensionality($Q$) are also given.}
	\label{tab:realdata2}
	\centering
	\begin{tabular}{|c|c|c|c|c|c|c|}
		\hline
			 & $N$ & $Q$ & AddNN & MC dropout\cite{Gal16} &DGP 5 \cite{Sal17} & PBP\cite{hernandez2015probabilistic}\\
		\hline
		Boston & $506$ & $13$& $3.03\pm 0.12$ & $2.97 \pm 0.19$ & $2.92\pm 0.17$ & $3.01\pm 0.18$ \\
		\hline
		Concrete &$1030$ &$8$ & $5.18 \pm 0.14$ & $5.23 \pm 0.12$ & $5.65\pm 0.10$ & $5.67 \pm 0.09$\\
		\hline
		Energy & $768$ & $8$ & $0.65 \pm 0.03$ & $1.66 \pm 0.04$ & $0.47\pm 0.01$ & $1.80\pm 0.05$\\
		\hline
		Kin8nm & $8192$ & $8$ & $0.07 \pm 0.00$ & $0.10 \pm 0.00$ & $0.06\pm 0.00$ & $0.10\pm 0.00$\\
		\hline
		Naval & $11934$ & $16$ & $0.01 \pm 0.00$ & $0.01 \pm 0.00$&$0.00\pm 0.00$&$0.01\pm 0.00$\\
		\hline
		Power & $9568$ &$4$ & $4.04\pm 0.03$ & $4.02 \pm 0.04$ &$3.68\pm 0.03$ &$4.12 \pm 0.03$\\
		\hline
		Protein & $45730$ &$9$ & $4.07\pm 0.01$ & $4.36\pm 0.01$&$3.72\pm 0.04$&$4.73\pm 0.01$\\
		\hline
		Wine & $1599$ & $11$ & $0.66\pm 0.01$ & $0.62 \pm 0.01$ &$0.63\pm 0.01$ &$0.64\pm 0.01$\\
		\hline
	\end{tabular}
\end{table}

\end{document}